\documentclass[11pt,letterpaper]{article}

\usepackage{mathtools}
\usepackage[round]{natbib}
\usepackage{amsmath,amsthm,amssymb,amsbsy,amscd,bm}
\usepackage{graphicx}
\usepackage{booktabs}
\usepackage{deepthink}
\usepackage[capitalize,noabbrev]{cleveref}

\theoremstyle{plain}
\newtheorem{theorem}{Theorem}[section]
\newtheorem{proposition}[theorem]{Proposition}
\newtheorem{lemma}[theorem]{Lemma}

\theoremstyle{definition}

\theoremstyle{remark}

\newcommand{\R}{\mathbb{R}}
\newcommand{\mW}{\bm W}
\newcommand{\paren}[1]{\left(#1\right)}

\title{Compressible Dynamics in Deep Overparameterized Low-Rank Learning \& Adaptation}

\authorblock{
  \textbf{Can Yaras},
  \textbf{Peng Wang},
  \textbf{Laura Balzano},
  \textbf{Qing Qu}
}

\affiliation{
  Department of Electrical Engineering \& Computer Science, University of Michigan
}

\abstracttext{
\noindent While overparameterization in machine learning models offers great benefits in terms of optimization and generalization, it also leads to increased computational requirements as model sizes grow. In this work, we show that by leveraging the inherent low-dimensional structures of data and compressible dynamics within the model parameters, we can reap the benefits of overparameterization without the computational burdens. In practice, we demonstrate the effectiveness of this approach for deep low-rank matrix completion as well as fine-tuning language models. Our approach is grounded in theoretical findings for deep overparameterized low-rank matrix recovery, where we show that the learning dynamics of each weight matrix are confined to an invariant low-dimensional subspace. Consequently, we can construct and train compact, highly compressed factorizations possessing the same benefits as their overparameterized counterparts. In the context of deep matrix completion, our technique substantially improves training efficiency while retaining the advantages of overparameterization. For language model fine-tuning, we propose a method called ``Deep LoRA'', which improves the existing low-rank adaptation (LoRA) technique, leading to reduced overfitting and a simplified hyperparameter setup, while maintaining comparable efficiency. We validate the effectiveness of Deep LoRA on natural language tasks, particularly when fine-tuning with limited data.
}
\date{\today}
\resources{\href{https://github.com/cjyaras/deep-lora-transformers}{Code Repository}}
\correspondence{cjyaras@umich.edu}

\begin{document}

\makeDeepthinkHeader

\begin{figure}[h!]
\centering
\vspace{-2em}
\includegraphics[width=0.32\linewidth]{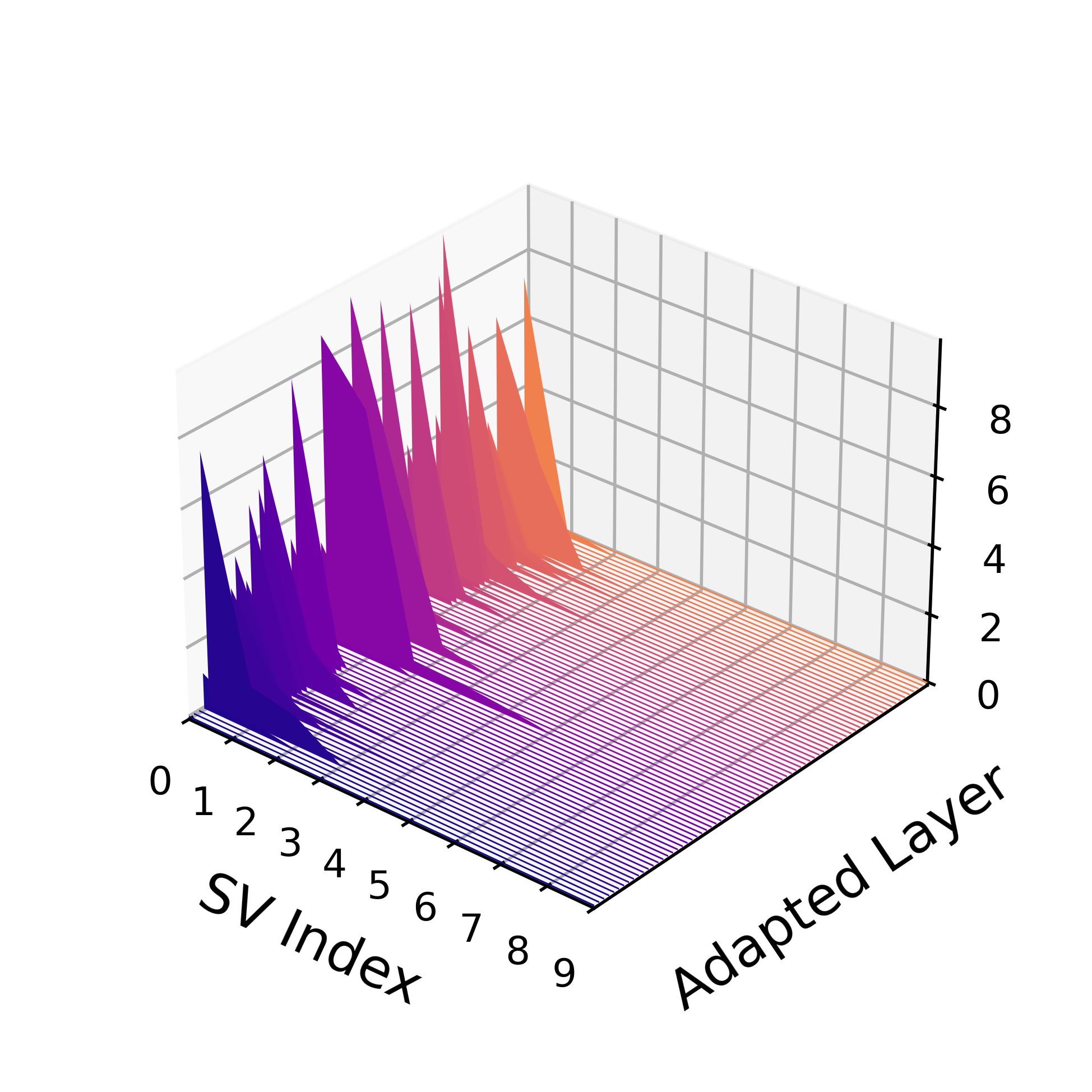}
\includegraphics[width=0.32\linewidth]{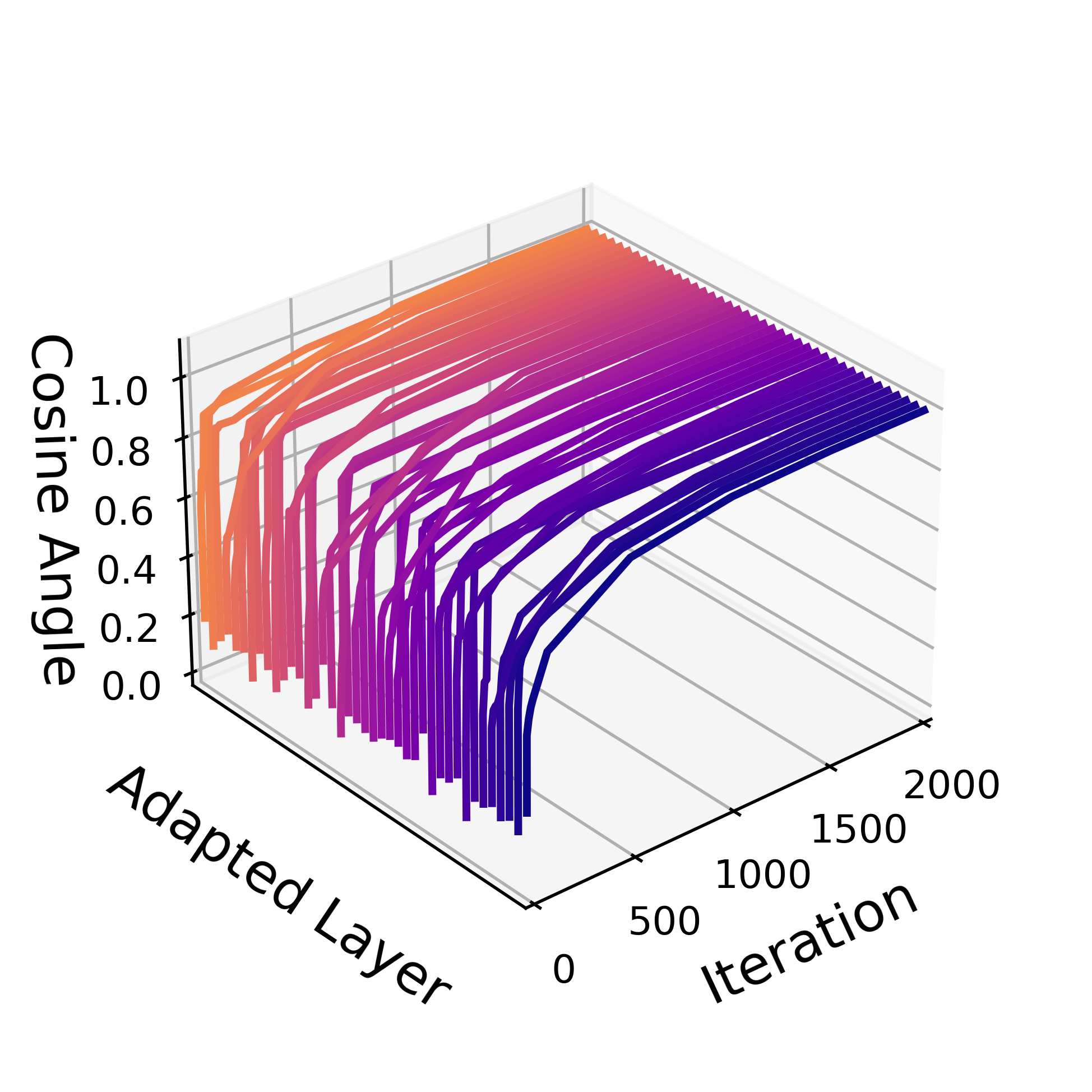}
\includegraphics[width=0.32\linewidth]{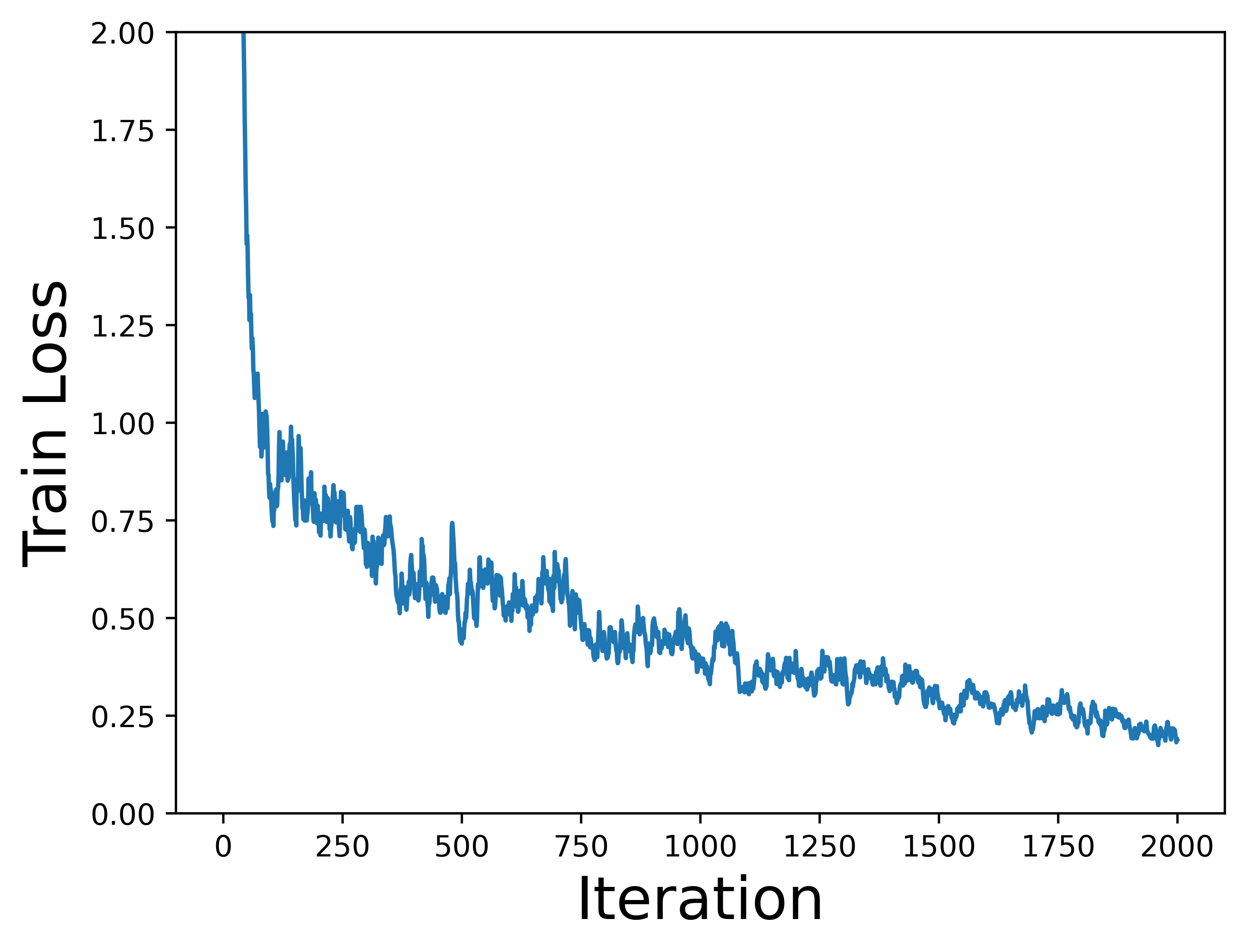}
\caption{\textbf{Invariant low-dimensional subspaces in deep overparameterized adaptation of language models.} Fine-tuning BERT \citep{devlin2019bert} with deep overparameterized adaptation on the STS-B dataset \citep{cer2017semeval}. \textit{Left}: \textbf{Singular value} spectra across all adapted layers at the end of fine-tuning. \textit{Middle}: \textbf{Alignment of subspaces} formed by top 8 right singular vectors between current adapted weights and final adapted weights throughout training. \textit{Right}: \textbf{Training loss} continues to decrease in iterations after subspace alignment with final adapted weights. See \Cref{sec:dclora} for more details.}
\label{fig:low_rank_subspace}
\end{figure}

\newpage
\tableofcontents
\newpage

\section{Introduction}\label{sec:intro}
In recent years, there has been a growing interest within the realm of deep learning in \textit{overparameterization}, which refers to employing a greater number of model parameters than necessary to interpolate the training data. While this may appear counterintuitive initially due to the risk of overfitting, it has been demonstrated to be an effective modeling approach \citep{zhang2021understanding,wu2017towards,allen2019learning,buhai2020empirical,xu2018benefits}, primarily attributed to improved optimization landscape and implicit algorithmic regularization. In the context of large language models (LLMs) \citep{radford2019language,brown2020language}, empirical scaling laws \citep{kaplan2020scaling} suggest that larger models are more sample efficient, often requiring fewer samples to reach the same test loss. 

Taking the problem of low-rank matrix recovery as an illustrative example, the seminal work of \citet{arora2019implicit} showed that deeper factorizations better promote low-rank solutions as a function of depth, consequently mitigating overfitting in the overparameterized regime compared to a classical two-layer factorization approach; see \Cref{fig:depth_2_v_3} (left). On the other hand, increasing the width of each layer substantially reduces the number of iterations to reach the same training error; see \Cref{fig:depth_2_v_3} (right).

\begin{figure}[h]
\centering
\includegraphics[width=0.8\linewidth]{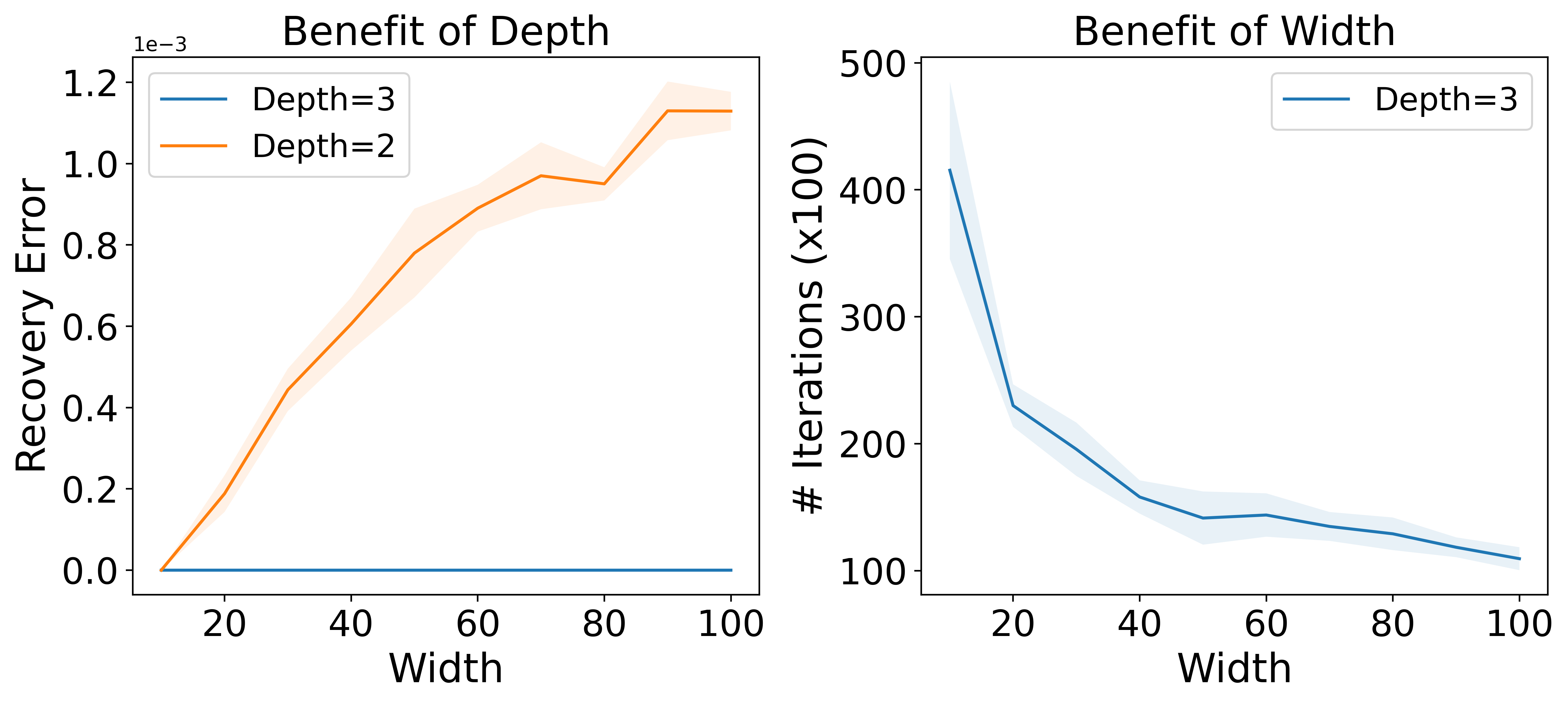}
\caption{\textbf{Benefits of depth \& width in overparameterized matrix completion} with $d=100$, $r^*=5$, $\epsilon_l = 10^{-3}$ and 30\% of entries observed. \textit{Left}: Recovery error vs. width for shallow and deep factorizations. \textit{Right}: Number of GD iterations to converge to $10^{-10}$ error vs. width. We observe that depth prevents overfitting, while width improves convergence.}
\label{fig:depth_2_v_3}
\end{figure}

While overparameterization offers remarkable benefits, it also comes with its computational challenges. The significantly increased number of parameters inevitably results in dramatically higher computational costs. This naturally raises a fundamental question: can we attain the benefits of overparameterization with a substantial reduction in computational costs?

In this work, we show that we can achieve this by exploiting low-dimensional structures of data and compressible learning dynamics in the model weights. In the context of low-rank matrix recovery via deep overparameterized factorizations, we discover an interesting phenomenon that for each weight matrix, \emph{the learning dynamics only happen within an approximately invariant low-dimensional subspace} throughout all iterations. We rigorously prove this for deep matrix factorization, which also allows us to compress the number of training parameters significantly when dealing with deep matrix completion. Consequently, we can construct and train a nearly equivalent, yet much smaller, compressed factorization without sacrificing the advantages of its overparameterized counterpart.

Interestingly, we empirically find that the above phenomenon can also be observed when employing deep overparameterized weight updates for fine-tuning language models; see \Cref{fig:low_rank_subspace} for an illustration. Therefore, we can adapt our idea of compressing deep matrix factorization to improve language model fine-tuning. For fine-tuning large-scale pretrained language models, recently low-rank adaptation (LoRA) stands out as the most commonly-used technique due to its effectiveness and efficiency \citep{hu2021lora}. The basic idea of LoRA is to freeze the pretrained weights and adapt each one to new tasks by adding and optimizing an update in the form of a two-layer low-rank decomposition. Nonetheless, in practical scenarios, selecting the optimal rank of the decomposition can pose a significant challenge. If the rank is not chosen properly, it may lead to overfitting, particularly when we overestimate the rank or when there is limited downstream data available.

We deal with this drawback of LoRA by employing a deep (three-layer) overparameterized factorization for the trainable update, which is constructed and optimized via the compression technique used for deep matrix completion. As such, our new method, which we term as \emph{Deep LoRA}, enjoys notable advantages over the original LoRA method, namely (i) \textbf{less overfitting} by exploiting depth, and (ii) \textbf{fewer hyperparameters} without rank $r$ and scale $\alpha$ having to be carefully tuned across all layers, all while having a comparable parameter efficiency due to compression.

\paragraph{Contributions.} We summarize our contributions below.
\begin{itemize}[leftmargin=*]
  \item \textbf{Practical contributions.} We develop efficient compression methods by exploring compressible learning dynamics in overparameterized factorizations. Our method enjoys the benefits of overparameterization while significantly improving its efficiency. We demonstrate the effectiveness not only on deep matrix completion, but also for improving LoRA for language model fine-tuning.
 \item \textbf{Theoretical contributions.} Our methods are inspired by our theoretical results for deep matrix factorization. Mathematically, we rigorously prove the existence of invariant low-dimensional subspaces throughout gradient descent for each weight matrix, and show how they can constructed in practice. 
\end{itemize}

\paragraph{Related Works}
There is a great deal of literature on implicit regularization in the setting of matrix factorization/linear networks \citep{neyshabur2015search,gunasekar2017implicit,arora2019implicit,moroshko2020implicit,timor2023implicit,ji2019gradient,gidel2019implicit,you2020robust,liu2022robust}, as well as low-rank learning in deep networks \citep{jaderberg2014speeding,sainath2013low,denil2013predicting,khodak2020initialization,oymak2019generalization,pmlr-v238-min-kwon24a,tarzanagh2023transformers}. Similarly, there is an abundance of work discussing the benefits of overparameterization \citep{du2019width,arora2018optimization,allen2019convergence,arpit2019benefits}.

\section{Warm-up Study: Deep Matrix Factorization}
\label{sec:theory}
Towards gaining theoretical insights into the phenomena in \Cref{fig:low_rank_subspace}, we first build some intuition based on the problem of deep matrix factorization. Under simplified settings, we rigorously unveil the emergence of low-dimensionality and compressibility in gradient descent learning dynamics.

\setlength{\belowdisplayskip}{4pt} 
\setlength{\abovedisplayskip}{4pt} 

\subsection{Basic Setup}\label{subsec:setup}
Given a low-rank matrix $\bm \Phi \in \R^{d_x \times d_y}$ with $\mbox{rank}(\bm \Phi) = r^*$, we approximate the matrix $\bm \Phi$ by an $L$-layer deep overparameterized factorization  
\begin{equation}
    f(\bm \Theta) := \mW_L \mW_{L-1} \cdots \mW_2 \mW_1 = \mW_{L:1}, \label{eqn:deep-matrix}
\end{equation}
where $\bm \Theta = (\mW_l)_{l=1}^L$ are the parameters with weights $\mW_l \in \mathbb R^{d_l \times d_{l-1}}$ for $l \in [L]$. We consider the case where the weights are all square $d_0 = d_1 = \dots = d_L = d$, and learn the parameters $\bm \Theta$ by solving
\begin{equation}\label{eq:l2_loss}
    \min_{\bm \Theta} \ell(\bm \Theta) = \frac{1}{2}\|f(\bm \Theta) - \bm \Phi\|_F^2
\end{equation}
via gradient descent (GD) from scaled \emph{orthogonal} initialization, i.e., we initialize parameters $\bm \Theta(0)$ such that 
\begin{equation}\label{eq:init}
    \mW_l(0) \mW_l(0)^\top = \mW_l(0)^\top \mW_l(0) = \epsilon_l^2 \bm I_d, \; l \in [L]
\end{equation}
where $\epsilon_l > 0$. We assume this for ease of analysis, and believe that our results could hold for arbitrary \textit{small} initialization.
For each weight matrix, the GD iterations can be written as
\begin{equation}\label{eq:gd}
    \mW_l(t+1) = (1 - \eta \lambda) \mW_l(t) - \eta \nabla_{\mW_l} \ell(\bm \Theta(t)), \; l \in [L]
\end{equation}
for $t = 0, 1, 2, \dots$,
where $\eta > 0$ is the learning rate and $\lambda \geq 0$ is an optional weight decay parameter.


\subsection{Main Theorem}\label{subsec:main-thm}
We show that learning only occurs within an invariant low-dimensional subspace of the weight matrices, whose dimensionality depends on $\mbox{rank}(\bm \Phi)$.

\begin{figure}
\centering
\vspace{-4.5em}
\includegraphics[width=\linewidth]{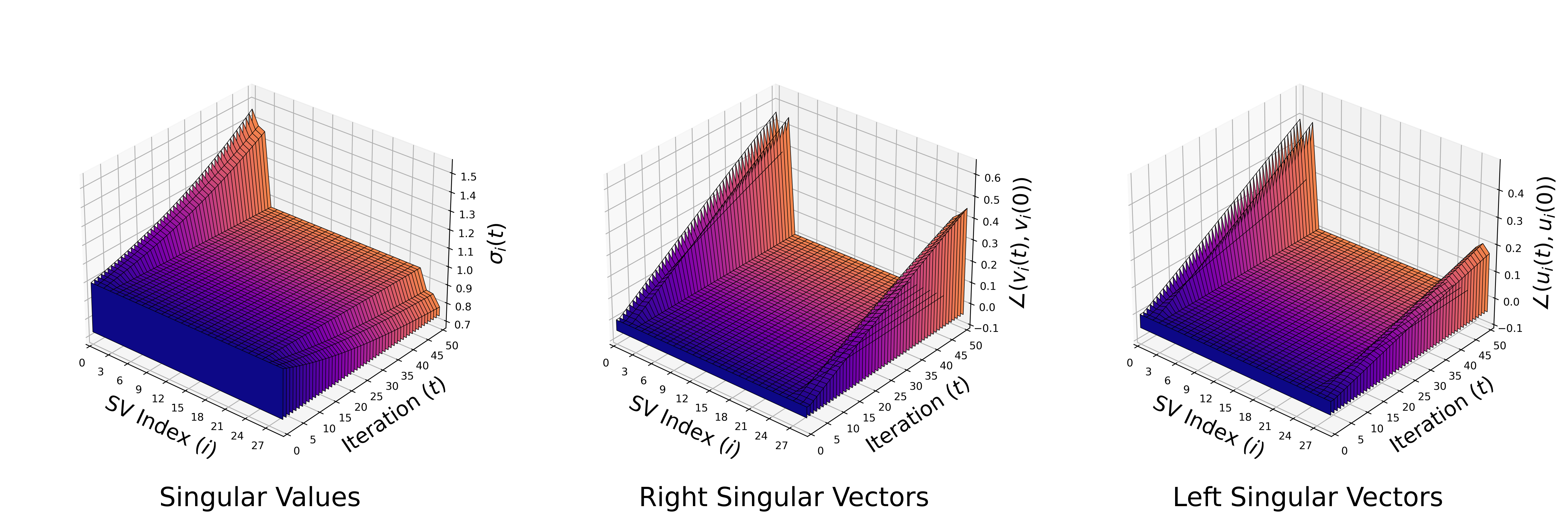}
\vspace{-1em}
\caption{\textbf{Evolution of SVD of weight matrices.} We visualize the SVD dynamics of the first layer weight matrix of an $L=3$ layer deep matrix factorization for a random matrix with $d = 30$, $r^*=3$, $\epsilon_l = 1$ throughout GD without weight decay. \textit{Left}: Magnitude of the $i$-th singular value $\sigma_i(t)$ at iteration $t$. \textit{Middle}: Angle $\angle(\bm v_i(t), \bm v_i(0))$ between the $i$-th right singular vector at iteration $t$ and initialization. \textit{Right}: Angle $\angle(\bm u_i(t), \bm u_i(0))$ between the $i$-th left singular vector at iteration $t$ and initialization.}
\label{fig:svd}
\end{figure}

\begin{theorem}\label{thm:1} Let $\mW_l(t)$ satisfy the initialization scheme \eqref{eq:init} and updates \eqref{eq:gd}, and suppose $\bm \Phi \in \R^{d \times d}$ is at most rank $r$ and let $m := d - 2r > 0$. Then there exist orthogonal matrices $(\bm U_l)_{l=1}^L \subset \mathcal O^{d\times d}$ and $(\bm V_l)_{l=1}^L \subset \mathcal O^{d\times d}$ (depending only on $\bm \Theta(0)$ and $\bm \Phi$) satisfying $\bm V_{l+1} = \bm U_l$ for $l \in [L-1]$, such that $\mW_l(t)$ admits the decomposition
\begin{equation}\label{eq:weight_structures}
    \mW_l(t) = \bm U_l
    \begin{bmatrix}
        \widetilde{\mW}_l(t) & \bm{0} \\
        \bm{0} & \rho_l(t) \bm{I}_m
    \end{bmatrix} 
    \bm V_l^\top
\end{equation}
for all $l \in [L]$ and $t \geq 0$, where $\widetilde{\mW}_l(t) \in \R^{2r \times 2r}$ with $\widetilde{\mW}_l(0) = \epsilon_l \bm I_{2r}$, and
\begin{equation}\label{eq:rho}
    \rho_l(t) = \rho_l(t-1)\cdot (1 - \eta \lambda - \eta \cdot \prod_{k\neq l} \rho_k^2(t-1))
\end{equation}
for all $l \in [L]$ and $t \geq 1$ with $\rho_l(0) = \epsilon_l$.
\end{theorem}
In the following, we discuss several implications of our result and its relationship to previous work. 

\begin{itemize}[leftmargin=*]
\item \textbf{SVD dynamics of weight matrices.} The decomposition \eqref{eq:weight_structures} is closely related to the singular value decomposition (SVD) of $\bm W_l(t)$. Specifically, let $\bm U_l = [\bm U_{l,1}\ \bm U_{l,2}]$, $\bm V_l = [\bm V_{l,1}\ \bm V_{l,2}]$, where $\bm U_{l,1},\bm V_{l,1} \in \mathcal O^{d\times 2r}$, $\bm U_{l,2},\bm V_{l,2} \in \mathcal O^{d\times (d-2r)}$. Let $\widetilde{\bm W}_l(t) = \widetilde{\bm U}_l(t)\widetilde{\bm \Sigma}_l(t)\widetilde{\bm V}_l^\top(t)$ be an SVD of $\widetilde{\bm W}_l(t)$, where $\widetilde{\bm U}_l(t),\widetilde{\bm V}_l(t) \in \mathcal O^{2r}$ and $\widetilde{\bm \Sigma}_l(t) \in \R^{2r\times 2r}$ is a diagonal matrix. Then, by \eqref{eq:weight_structures} we can write $\bm W_l(t)$ as
\begin{align*}
\bm W_l(t) = 
\begin{bmatrix}
  \bm U_{l,1}\widetilde{\bm U}_l(t)  &  \bm U_{l,2}
\end{bmatrix} \begin{bmatrix}
        \widetilde{\bm \Sigma}_l(t) & \bm{0} \\
        \bm{0} & \rho(t) \bm{I}_m
    \end{bmatrix}  
\begin{bmatrix}
  \bm V_{l,1} \widetilde{\bm V}_{l}(t)  &  \bm V_{l,2}
\end{bmatrix}^\top
\end{align*}
which is essentially an SVD of $\bm{W}_l(t)$ (besides the ordering of singular values). According to this, we can verify that $\rho(t)$ is a (repeated) singular value undergoing minimal changes across iterations illustrated in \Cref{fig:svd} (left). Additionally, these repeated singular values correspond to \textit{invariant} subspaces $\bm U_{l, 2}, \bm V_{l, 2}$ that are stationary throughout GD, as seen in \Cref{fig:svd} (middle and right).

\item \textbf{Low-rank bias.} From \eqref{eq:rho}, we can show under mild assumptions that the GD trajectory for each weight matrix either remains or tends towards a solution with rank at most $2r$. This is true whether we employ implicit or explicit regularization. Indeed, if we use small initialization $\epsilon_l \approx 0$ with no weight decay $\lambda = 0$, then the fact that $\rho_l$ is a decreasing sequence (w.r.t. iteration) implies that the approximate rank of $\mW_l(t)$ can be no more than $2r$ throughout the entire trajectory. On the other hand, if we use weight decay with $\lambda > 0$, then we have $\rho_l(t) \rightarrow 0$ as $t \rightarrow \infty$. This forces $\mW_l(t)$ towards a solution of rank at most $2r$ when the training converges. See \Cref{app:learning_rate} for a formal statement and proof. This result is consistent with previous findings on low-rank and simplicity bias in deep networks \citep{huh2022low,galanti2022sgd,li2020towards,chou2024gradient}.
\item \textbf{Comparison to prior arts.} In contrast to existing work studying implicit bias of GD towards low-rank solutions \citep{gunasekar2017implicit,arora2019implicit}, our result explicitly shows how GD finds these solutions. Moreover, unlike previous work on implicit bias \citep{min2021explicit,gissin2019implicit,arora2019implicit,vardi2021implicit}, we also examine the effect of weight decay, which is commonly employed during the training of deep networks. Our analysis is distinct from that of \citep{saxe2014exact,saxe2019mathematical}, which studied continuous time dynamics under the special (separable) setting $\mW_{L:1}(0) = \bm U \bm V^\top$ with $\bm \Phi = \bm U \bm \Sigma \bm V^\top$. In comparison, our result applies to discrete time dynamics and holds for initialization that is agnostic to the target matrix. It should also be noted that our result does not depend on \emph{balanced} initialization like those in \citep{arora2018convergence}, as the initialization scale $\epsilon_l$ for each layer can be arbitrarily different from one another. 
\end{itemize}

\paragraph{A sketch of analysis.} We now provide a rough sketch for the beginning of the proof of \Cref{thm:1} in the special case of small initialization $\epsilon_l = \epsilon \approx 0$ for all $l \in [L]$ and $\lambda = 0$, highlighting the construction of the invariant subspace at initialization. The full proof can be found in \Cref{app:proof_thm1}.

\begin{proof}[Proof sketch]
Since $\epsilon^L \approx 0$, from the gradient of $\ell(\bm \Theta)$ (see \Cref{app:proofs}), we have 
\begin{equation}
\label{eq:g1}
    \bm G_1 := \nabla_{\bm W_1} \ell(\bm \Theta(0)) \approx - \mW_{L:2}^\top(0) \bm \Phi
\end{equation}
implying that the rank of $\bm G_1$ is (approximately) at most $r$. Now consider the subspace $\mathcal{S} = \mathcal{N}(\bm G_1) \cap \mathcal{N}(\bm G_1^\top \mW_1(0))$, where we have $\dim\,\mathcal{S} \geq d-2r$. Then, there exist orthonormal sets $\{\bm v_i\}_{i=1}^{d-2r}$ and $\{\bm u_i\}_{i=1}^{d-2r}$ which satisfy $\bm G_1 \bm v_i = \bm 0$, $\bm u_i \propto \bm W_1(0) \bm v_i$ and therefore
\begin{equation*}
    \bm G_1^\top \bm u_i \propto \bm G_1^\top \bm W_1(0) \bm v_i = \bm 0
\end{equation*}
so along with the orthogonality of $\bm W_1(0)$, the pairs $(\bm u_i, \bm v_i)$ form singular vector pairs of both $\bm W_1(0)$ and $\bm W_1(1)$ simultaneously as they remain unchanged by the gradient update $\bm G_1$, giving the last $d-2r$ columns of $\bm V_1$ and $\bm U_1$ respectively. To see that we can take $\bm V_2 = \bm U_1$, for instance, we note that
\begin{align*}
    \nabla_{\mW_2} \ell(\bm \Theta(0)) \cdot \bm u_i
    &\approx - \mW_{L:3}^\top(0) \bm \Phi \mW_1^\top(0) \bm u_i \\
    &\propto \mW_{L:3}^\top(0) \bm \Phi \bm v_i = \bm 0
\end{align*}
by $\eqref{eq:g1}$, showing that $\bm u_i$ are invariant under gradient updates in the second layer.
\end{proof}

\subsection{Compression of Overparameterized Factorization}\label{subsec:compression}

We now show that, as a consequence of \Cref{thm:1} and the proof sketch, we can run GD on dramatically \emph{fewer} parameters to achieve a near \emph{identical} end-to-end trajectory as the original (full-width) factorization; see \Cref{fig:equiv_traj_l2}.

\paragraph{Constructing the ``equivalent'' compressed factorization.} More specifically, given that $\bm \Phi$ is at most rank $r$ and $d-2r > 0$, from \Cref{thm:1} we observe that
 \begin{align}
     \mW_{L:1}(t) & = \bm U_{L, 1} \widetilde{\mW}_{L:1}(t)  \bm V_{1, 1}^\top + \paren{\prod_{l=1}^L\rho_l(t)}\cdot \bm U_{L,2} \bm V_{1,2}^\top \nonumber  \\
     & \approx \underbrace{\bm U_{L, 1} \widetilde{\mW}_{L:1}(t)  \bm V_{1, 1}^\top}_{ =:  f_C( \widetilde{\bm \Theta}, \bm U_{L, 1}, \bm V_{1, 1})}, \quad \forall\;t=1,2,\dots, \label{eqn:approx}
 \end{align}
 when we use initialization of small scale (i.e., $(\epsilon_l)_{l=1}^L$ are small). Here, $\widetilde{\mW}_{L:1} =  \widetilde{\mW}_{L} \widetilde{\mW}_{L-1} \cdots \widetilde{\mW}_1 $ with compressed weights $\widetilde{\mW}_l \in \R^{2r\times 2r}$. Correspondingly, $f_C( \widetilde{\bm \Theta}, \bm U_{L, 1}, \bm V_{1, 1})$ denotes the compressed function with compressed parameters $\widetilde{\bm \Theta} = (\widetilde{\mW}_l)_{l=1}^L$. As such, we can expect that solving
 \begin{equation}\label{eq:compression_loss}
   \min_{\widetilde{\bm \Theta} } \ell_C (\widetilde{\bm \Theta}) = \frac{1}{2}\| f_C (\widetilde{\bm \Theta}, \bm U_{L, 1}, \bm V_{1, 1}) - \bm \Phi\|_F^2 
\end{equation}
will approximately give the same solution as \eqref{eq:l2_loss}.

\begin{figure}[t]
\centering
\vspace{-4.5em}
\includegraphics[width=0.7\linewidth]{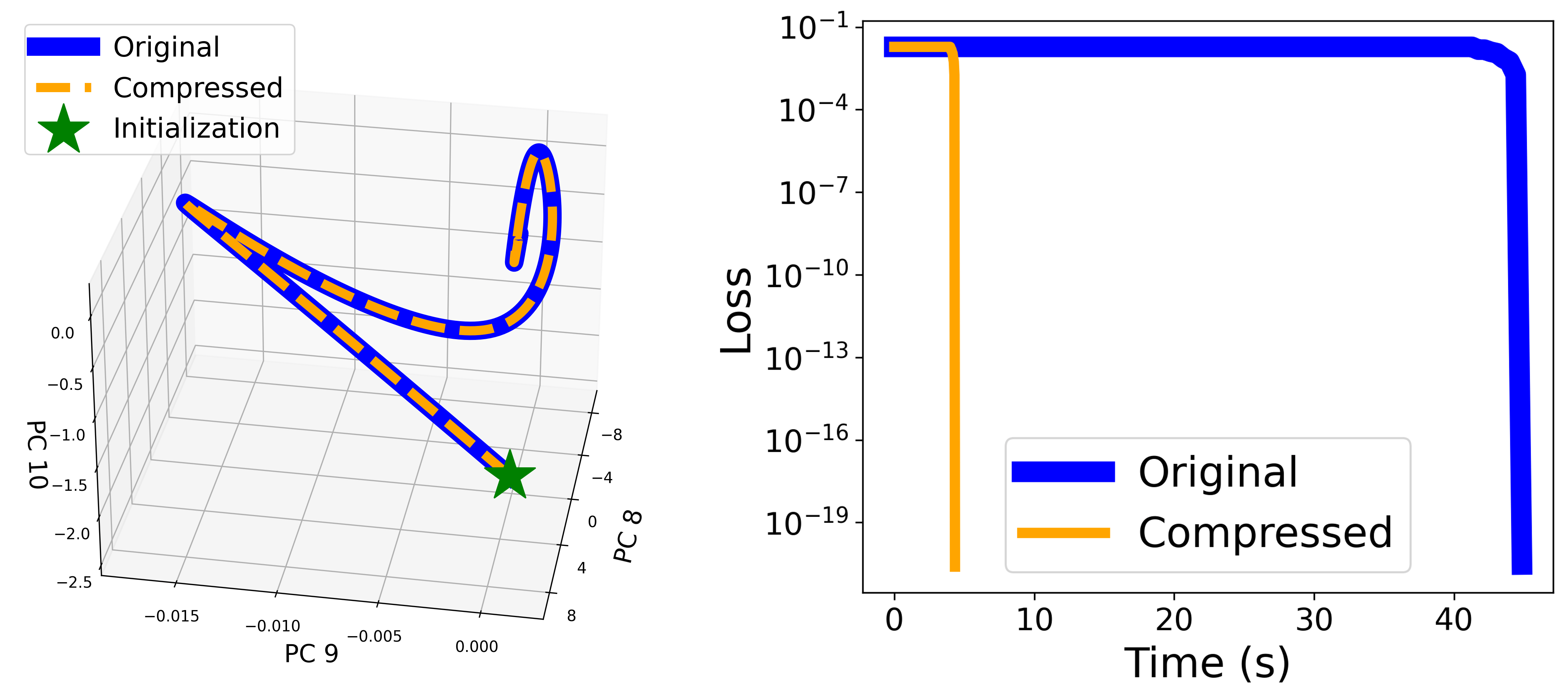}
\vspace{-1em}
\caption{\textbf{Network compression for deep matrix factorization.} Comparison of trajectories for optimizing the original problem \eqref{eq:l2_loss} vs. the compressed problem \eqref{eq:compression_loss} with $L=3$, $d=1000$, $r = r^* = 5$, and $\epsilon_l = 10^{-3}$. \textit{Left}: Principal components of end-to-end GD trajectories. \textit{Right}: Training loss vs. wall-time comparison.}
\label{fig:equiv_traj_l2}
\end{figure}
 
\paragraph{Constructing the factors $(\bm U_l, \bm V_l)_{l=1}^L$.} 
As \Cref{thm:1} only showed the existence of $(\bm U_l, \bm V_l)_{l=1}^L$, to solve \eqref{eq:compression_loss} via GD, we need a practical recipe for constructing $(\bm U_l, \bm V_l)_{l=1}^L$ efficiently \textit{at initialization} of small scale $\epsilon_l$. This can be achieved based upon our proof sketch in \Cref{subsec:main-thm}: we compute $\bm G_1 = \nabla_{\bm W_1} \ell(\bm \Theta(0)) \in \R^{d\times d}$, find an orthonormal set $\{\bm v_i\}_{i=1}^{d-2r}$ contained in $\mathcal{S} = \mathcal{N}(\bm G_1) \cap \mathcal{N}(\bm G_1^\top \mW_1(0))$, and complete to an orthonormal basis to yield $\bm V_1$. The remaining $\bm U_l, \bm V_l$ can then be iteratively constructed via
\begin{equation*}
    \bm U_l = \bm W_l(0) \bm V_l / \epsilon_l, \; \bm V_{l+1} = \bm U_l,\; l=1,\cdots,L-1,
\end{equation*}
and $\bm U_L = \bm W_L(0) \bm V_L / \epsilon_L$. Finally, we take the first $2r$ columns of $\bm U_L$ and $\bm V_1$ to yield $\bm U_{L, 1}$ and $\bm V_{1, 1}$, respectively. It should be noted that these compressed factors are related to, yet distinct from, spectral initialization, which is well-studied in the literature \citep{chi2019nonconvex,khodak2020initialization,stoger2021small}. Since $\bm U_{l,1}, \bm V_{l,1}$ are constructed via orthogonal complements to nullspaces involving the gradient, these directions do indeed correlate with the top singular subspaces of $\bm \Phi$ in the deep matrix factorization case (although we do not use the singular value information). On the other hand, our approach is more general through the lens of compression, as it can be applied to a given deep overparameterized factorization trained on an \emph{arbitrary} loss.

\paragraph{Optimization, complexity, and approximation error.}
In summary, we can approximately solve the original problem by solving 
\eqref{eq:compression_loss} via GD for the compressed parameters $\widetilde{\bm \Theta} = (\widetilde{\mW}_l)_{l=1}^L$, starting from small initialization ($\epsilon_l \approx 0$). The factors $\bm U_{L, 1}, \bm V_{1, 1}$ can be efficiently constructed based upon an iterative scheme that we discussed above from the initial weights. 

 Comparing the parameter counts of the compressed $f_C( \widetilde{\bm \Theta}, \bm U_{L, 1}, \bm V_{1, 1})$ vs. the original $ f(\bm \Theta)$, we only need to optimize $4L\cdot r^2$ parameters compared to the original $L\cdot d^2$. Since $r \ll d$, our approach leads to significant improvement in efficiency during GD; see \Cref{fig:equiv_traj_l2} (right). On the other hand, compression requires some additional computation to construct the factors $\bm U_{L, 1}$ and $\bm V_{1, 1}$ prior to training, which involves taking a gradient of the first weight in the original factorization followed by an SVD or QR decomposition to compute an orthonormal basis for $\mathcal{S}$. While this requires an additional $O(d^3)$ compute, this has the same complexity as a single iteration of GD for the original factorization and is therefore a negligible overhead when comparing the two.

Finally, the following result demonstrates that our compression method can achieve an almost identical end-to-end trajectory when we use small initializations; see \Cref{fig:equiv_traj_l2} (left).

\begin{proposition}\label{prop:1}
For $r$ such that $m := d-2r > 0$, if we run GD on the compressed weights $\widetilde{\bm \Theta}$ as described above for the loss \eqref{eq:compression_loss}, we have 
\begin{equation*}
    \left\| f(\bm \Theta(t)) - f_C(\widetilde{\bm \Theta}(t), \bm U_{L, 1}, \bm V_{1, 1}) \right\|_F^2 \leq m \cdot \prod_{l=1}^L \epsilon_l^2
\end{equation*}
for any iterate $t = 0, 1, 2, \cdots$. Here, $\epsilon_l$ is the initialization scale for the weight $\mW_l(0)$.
\end{proposition}
The key idea of \Cref{prop:1} is that GD is invariant under orthogonal transformations, and each factor $\widetilde{\mW}_l$ in the end-to-end factorization in \eqref{eq:compression_loss} is the result of an orthogonal transformation of $\mW_l$. Then, the approximation error $m \cdot \prod_{l=1}^L \epsilon_l^2$ is only due to the approximation we showed in \eqref{eqn:approx}. We defer the full proof to \Cref{app:proof_prop1}.

\section{Application I: Deep Matrix Completion}
\label{sec:dmc}

In this section, we show that we can generalize our method in \Cref{sec:theory} from vanilla matrix factorization to solving low-rank matrix completion problems \citep{candes2012exact,candes2010power,davenport2016overview} via compressed deep factorizations. Given a ground-truth $\bm \Phi \in \R^{d\times d}$ with rank $r^* \ll d$, the goal of low-rank matrix completion is to recover $\bm \Phi$ from only a few number of observations encoded by a mask $\bm \Omega \in \{0, 1\}^{d\times d}$.
Adopting a matrix factorization approach, we minimize the objective
\begin{equation}\label{eq:mc_loss}
    \ell_{\mathrm{mc}}(\bm \Theta) = \frac{1}{2}\|\bm \Omega \odot (f(\bm \Theta) - \bm \Phi)\|_F^2,
\end{equation}
where $f(\bm \Theta)$ is the deep overparameterized factorization introduced in \eqref{eqn:deep-matrix}. The problem simplifies to deep matrix factorization \eqref{eq:l2_loss} that we studied earlier when $\bm \Omega = \bm 1_d \bm 1_d^\top$ in the full observation case. Additionally, \eqref{eq:mc_loss} reduces to vanilla (shallow) matrix factorization when $L=2$, whose global optimality and convergence have been widely studied under various settings \citep{jain2013low,zheng2016convergence,sun2016guaranteed,ge2016matrix,bhojanapalli2016global,ge2017no,gunasekar2017implicit,li2019non,chi2019nonconvex,li2018algorithmic,soltanolkotabi2023implicit,sun2018geometric,zhang2020symmetry,ding2021rank}.

\begin{figure}[t]
  \centering
  \vspace{-4em}
  \includegraphics[width=\linewidth]{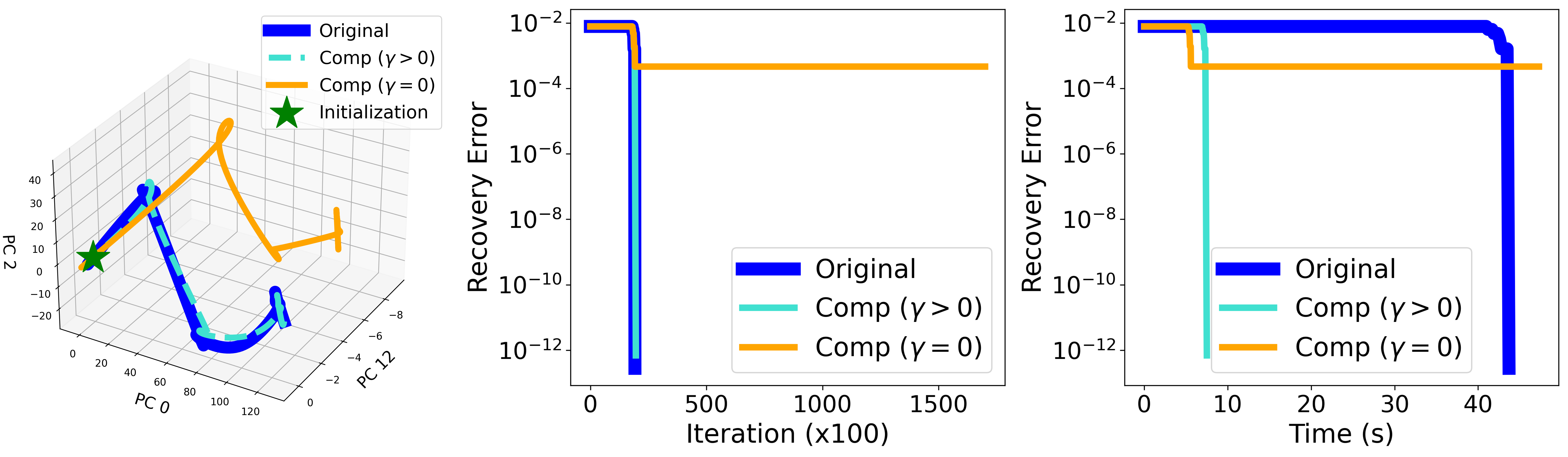}
  \caption{\textbf{Network compression for deep matrix completion.}
  Comparison of trajectories for optimizing the original problem \eqref{eq:mc_loss} vs. the compressed problem \eqref{eq:compression_loss_mc} with $\gamma$ discrepant updates ($\gamma = 0.01$) and ablating $\gamma$ ($\gamma = 0$) with $L=3$, $d=1000$, $r=r^*=5$, $\epsilon_l=10^{-3}$ and 20\% of entries observed. \textit{Left}: Principal components of end-to-end trajectories of each factorization. \textit{Middle}: Recovery error vs. iteration comparison. \textit{Right}: Recovery error vs wall-time comparison.}
  \label{fig:equiv_traj_mc}
\end{figure}

\paragraph{A double-edged sword of overparameterization.} 
 In practice, the true rank $r^*$ is not known -- instead, we assume to have an upper bound $r$ of the same order as $r^*$, i.e., $r^* \leq r \ll d$. Surprisingly,  overparameterization has
 advantages in terms of both depth $L$ and width $r$:
\begin{itemize}[leftmargin=*]
    \item \textbf{Benefits of depth: mitigating overfitting.} When $r > r^*$, it has been demonstrated \citep{arora2019implicit} that optimizing deeper factorizations (i.e., $L\geq 3$) generalize better in the low sample regime, while their shallow counterparts overfit; see \Cref{fig:depth_2_v_3} (left).
    \item \textbf{Benefits of width: improving convergence.}  On the other hand, increasing the width $r$ of the deep factorization beyond $r^*$ results in accelerated convergence in terms of iterations, see \Cref{fig:depth_2_v_3} (right). 
\end{itemize}
However, the advantages of overparameterization come with the challenges of much higher computational costs. For an $L$-layer factorization of (full)-width $d$, we require $O(L \cdot d^3)$ multiplications per iteration to evaluate gradients and need to store $O(L \cdot d^2)$ parameters, where $d$ is often very large. Using ideas from \Cref{sec:theory}, however, we can obtain the benefits of overparameterization without the extra computational costs.

\paragraph{Compression for deep matrix completion.} Given the similarity between deep matrix factorization and completion (i.e., $\bm \Omega = \bm 1_d \bm 1_d^\top$ vs arbitrary $\bm \Omega$), it seems straightforward to generalize our compression methods in \Cref{subsec:compression} to deep matrix completion. However, as shown by the orange trace in \Cref{fig:equiv_traj_mc}, direct application does not work well, as the compressed factorization's trajectory diverges from that of the original. This is because the compressed subspaces $\bm U_{L, 1}, \bm V_{1, 1} \in \R^{d \times 2\widehat{r}}$ computed at the initialization $\bm \Theta(0)$ via the gradient
\begin{equation*}
    \nabla_{\mW_1} \ell_{\mathrm{mc}}(\bm \Theta(0)) \approx - \mW_{L:2}^\top(0) [\bm \Omega \odot \bm \Phi]
\end{equation*}
can be \emph{misaligned} with the true subspace due to the perturbation by the mask $\bm \Omega$.

Nonetheless, this issue can be mitigated by slowly updating $\bm U_{L, 1}, \bm V_{1, 1}$ during training. Specifically, compared to \eqref{eq:compression_loss}, we minimize
\begin{align} \label{eq:compression_loss_mc}
    \min_{ \widetilde{\bm \Theta}, \bm U_{L,1},\bm V_{1,1} } \frac{1}{2}\|\bm \Omega \odot (f_C(\widetilde{\bm \Theta}, \bm U_{L, 1}, \bm V_{1, 1}) - \bm \Phi)\|_F^2 
\end{align}
via GD by updating $\widetilde{\bm \Theta}, \bm U_{L,1},\bm V_{1,1}$ simultaneously every iteration, with a learning rate $\eta$ on $\widetilde{\bm \Theta}$ along with a \emph{discrepant} learning rate $\gamma \eta$ on $\bm U_{L,1},\bm V_{1,1}$. Because we update $\bm U_{L,1},\bm V_{1,1}$ slower than updating $\widetilde{\bm \Theta}$, we generally choose $\gamma > 0$ to be small and tuned accordingly for the given problem.

As a result, we reduce computational costs to $O((L + d) \cdot r^2)$ multiplications per iteration for computing gradients, and $O(d \cdot r + L \cdot r^2)$ parameters. Yet still the trajectory of the deep compressed factorization ultimately aligns with that of the original, while converging roughly $5\times$ faster w.r.t. wall-time, as demonstrated in \Cref{fig:equiv_traj_mc}. Moreover, the accelerated convergence induced by the full-width trajectory results in the compressed factorization being $3\times$ faster than randomly initialized factorizations of similar width -- see \Cref{app:comp_random_dmc} for more details. 

\section{Application II: Model Fine-tuning}
\label{sec:dclora}

In this section, we show that our compression idea can be further extended to parameter-efficient fine-tuning of pretrained language models, specifically via low-rank adaptation (LoRA) \citep{hu2021lora}. In particular, inspired by our approach for deep matrix completion, we propose Deep Low-Rank Adaptation (Deep LoRA), which consistently outperforms \emph{vanilla} LoRA in the limited sample regime.

\subsection{Deep Low-Rank Adaptation (Deep LoRA)}

\paragraph{Background on LoRA.} With the ever-growing size of pretrained models and countless downstream tasks, full model fine-tuning is often computationally infeasible. Given a pretrained model whose parameters consist of a collection of dense weight matrices $\{\mW_{0k}\}_{k=1}^m \subset \R^{d\times d}$ (e.g., the query/key/value projections of a transformer \citep{vaswani2017attention}), LoRA seeks to adapt each layer to a given task by \emph{freezing} the pretrained weight $\{\mW_{0k}\}_{k=1}^m$ and optimizing an extra trainable low-rank factorization on top. In other words, the fine-tuned weight $\mW_k$ is given by
\begin{equation*}
    \mW_k = \mW_{0k} + \frac{\alpha}{r}\mW_k^{(2)} \mW_k^{(1)}
\end{equation*}
where $\alpha > 0$ is a tunable scale parameter and $\mW_k^{(2)} \in \R^{d \times r}$, $\mW_k^{(1)} \in \R^{r \times d}$ with $r \ll d$, thereby substantially reducing the number of trainable parameters during fine-tuning.

\begin{figure}[t]
\centering
\vspace{-4em}
\begin{minipage}{0.32\textwidth}
\centering
\includegraphics[width=\linewidth]{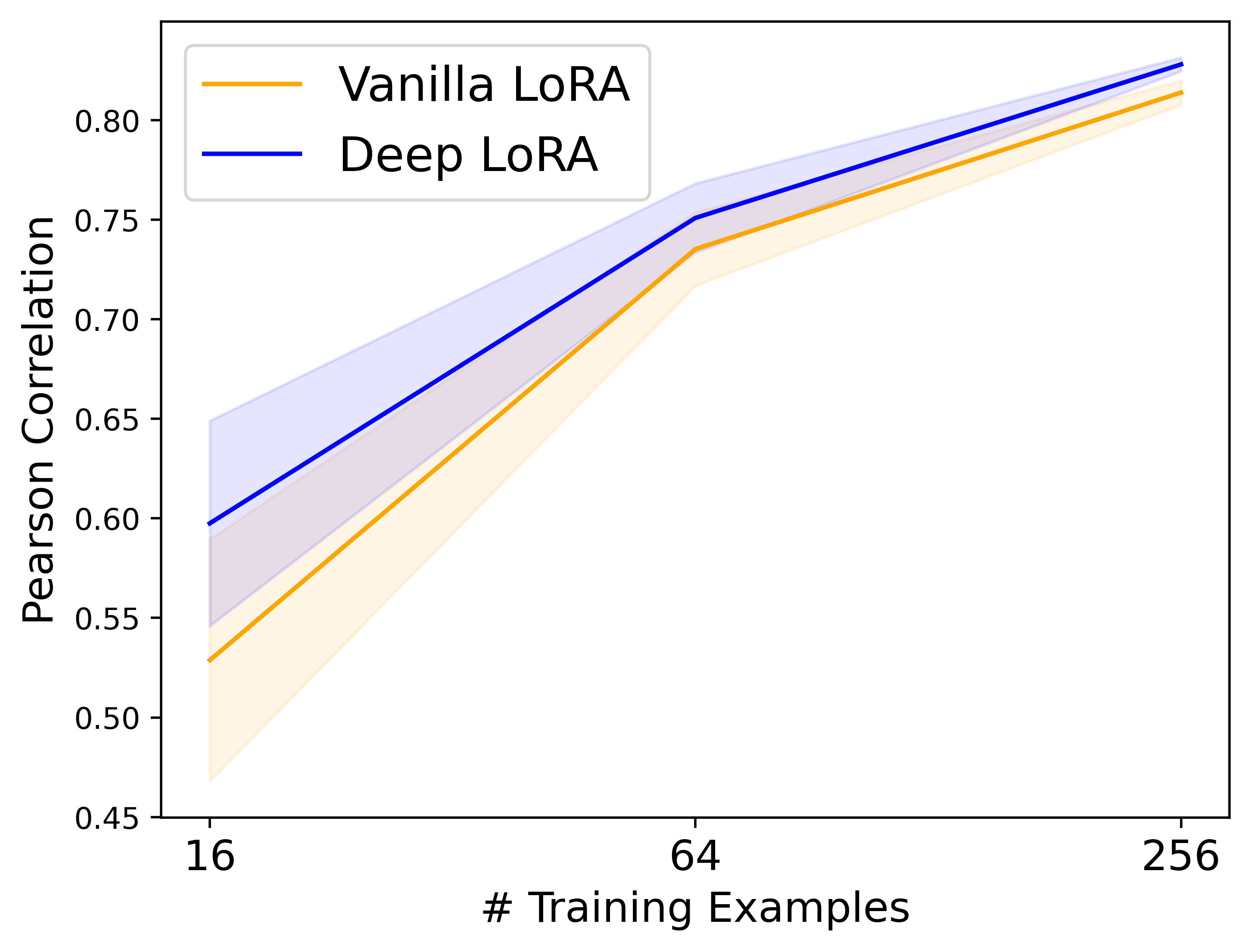}
\vspace{-0.2in}
\caption{\textbf{Deep LoRA shows better performance on few shot fine-tuning} over vanilla LoRA, with varying numbers of training samples. For each case, we draw $n$ samples at random from STS-B over 20 trials with different seeds, and measure performance on the validation split of each method using the same train set.}
\label{fig:fewshot_stsb}
\end{minipage}
\hfill
\begin{minipage}{0.32\textwidth}
\centering
\includegraphics[width=\linewidth]{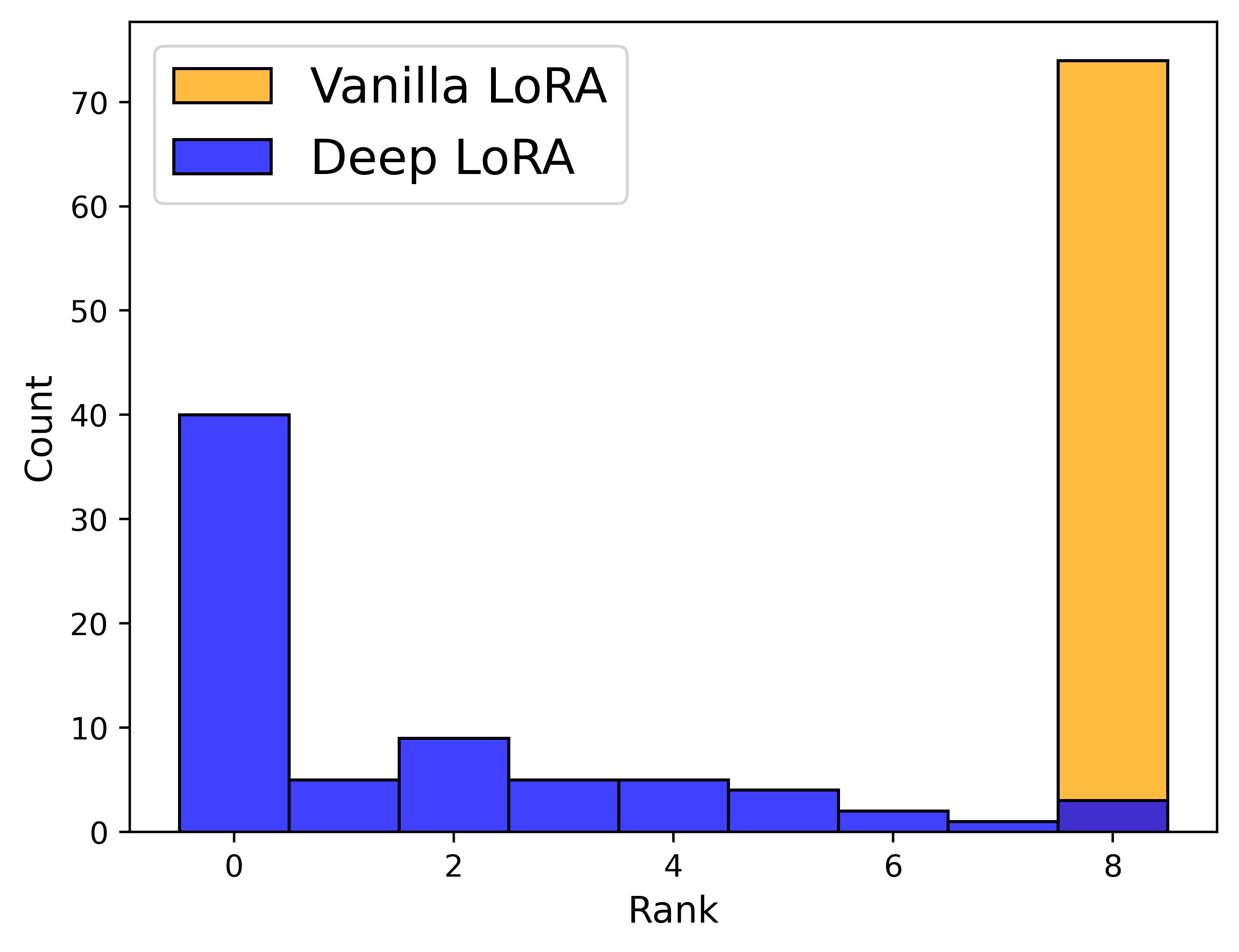}
\vspace{-0.2in}
\caption{\textbf{Deep LoRA finds lower rank solutions} compared to vanilla LoRA. We plot a histogram of numerical ranks for Deep LoRA and vanilla LoRA with $r=8$ after adapting to STS-B with 256 samples. The numerical rank is computed as the number of singular values $\sigma_i$ greater than $10^{-8}$ and $d\sigma_1 \epsilon $ where $\epsilon$ is machine epsilon.}
\label{fig:fewshot_256_ranks}
\end{minipage}
\hfill
\begin{minipage}{0.32\textwidth}
\centering
\vspace{-0.2in}
\includegraphics[width=\linewidth]{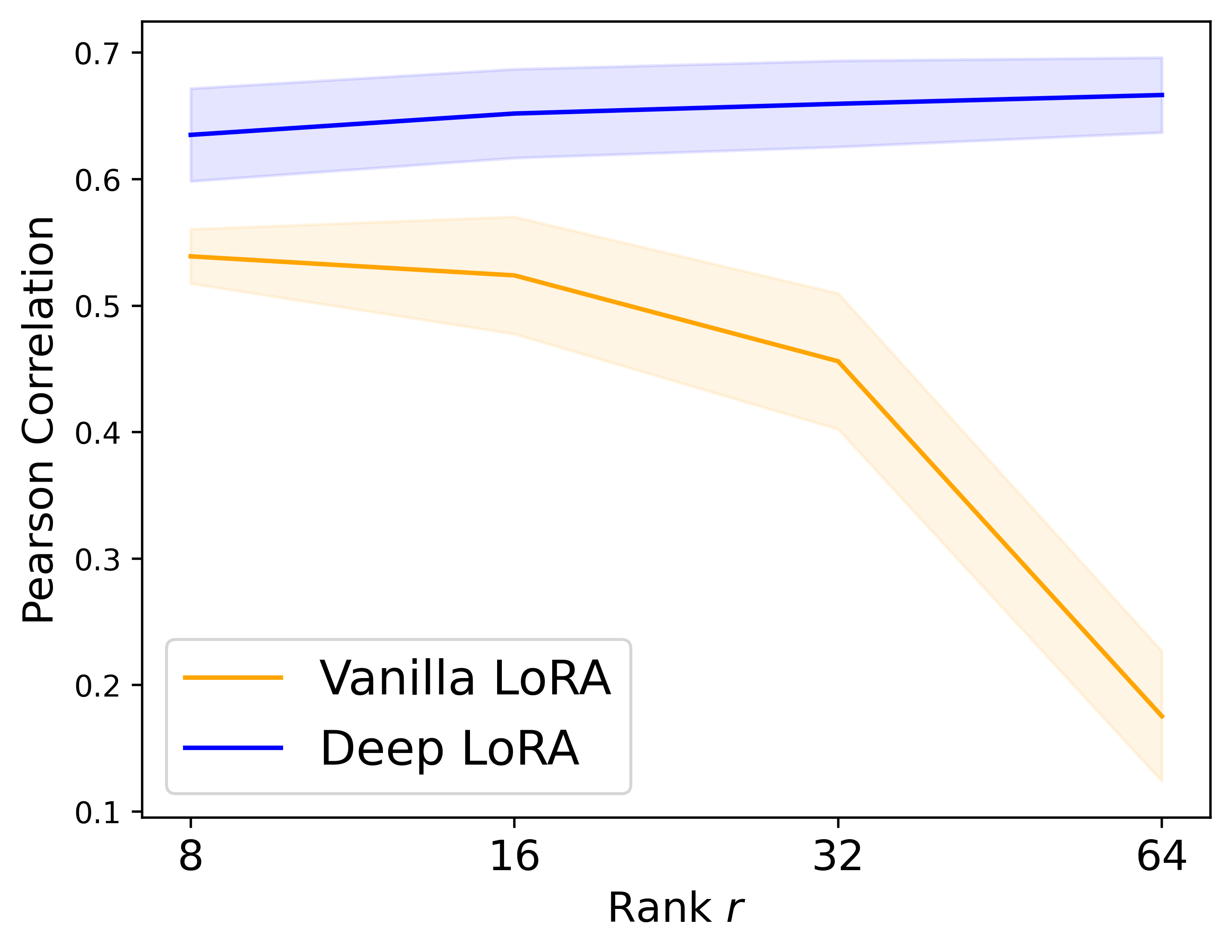}
\vspace{-0.2in}
\caption{\textbf{Deep LoRA is more robust to the choice of rank} compared to vanilla LoRA. For each choice of rank $r$, we draw 16 samples at random from STS-B over 5 trials with different seeds, and measure performance on the validation split of each method using the same train set.}
\label{fig:fewshot_stsb_varying_rank}
\end{minipage}
\end{figure}

\paragraph{Proposed method: Deep LoRA.} For vanilla LoRA, if we view adapting each weight matrix of model as an \emph{individual} low-rank factorization problem, we have demonstrated in previous sections that overparameterization and subsequently compressing factorizations improves generalization with minimal extra computational costs. With this in mind, we can employ a deep overparameterized adaptation of each pretrained weight as
\begin{equation}
\label{eq:deep_lora}
    \mW_k = \mW_{0k} + \underbrace{\mW_k^{(L)} \cdots \mW_k^{(2)} \mW_k^{(1)}}_{ =:\Delta \mW_k }
\end{equation}
where each $\mW_k^{(i)}$ is full-width, i.e., $\mW_k^{(i)} \in \R^{d \times d}$. Here, $L>0$ is the depth, and typically we choose $L=3$, which is precisely the setting of \Cref{fig:low_rank_subspace}. From the figure, we can see that (i) all the converged weights $\{\Delta \mW_k\}_{k=1}^m$ are very low-rank (left panel), (ii) the learning dynamics each for each weight \emph{approximately} stay within the same invariant subspace throughout the iterations (middle panel), and (iii) this happens independent of the training loss decreasing (right panel). 

These observations imply that deep overparameterized factorizations in Deep LoRA are \emph{highly compressible}, so we can apply the compression method from deep matrix completion in \Cref{sec:dmc} to compress the learning dynamics for each individual weight for model fine-tuning. Here, the major differences of our compression approach for deep LoRA from that of deep matrix completion is that (i) we have a separate compressed factorization for each layer to be adapted, and (ii) the fine-tuned loss function can be tailored for specific tasks (e.g., the cross-entropy) besides the $\ell_2$ loss. 

\paragraph{Advantages of Deep LoRA.}
Compared to vanilla LoRA, Deep LoRA has clear advantages that we highlight below. More details are provided in \Cref{sec:exp}.
\begin{itemize}[leftmargin=*]
    \item \textbf{Less overfitting in limited data regime.} Fine-tuning overparameterized models using LoRA can still result in overfitting in few shot or limited data regime \citep{SebastianRaschka_2023}. In comparison, the extra depth in \eqref{eq:deep_lora} of Deep LoRA can help prevent overfitting (see \Cref{fig:fewshot_stsb}), which is similar to deep matrix completion in \Cref{fig:depth_2_v_3}.
    \item \textbf{Robustness to the hyperparameter $r$.} As shown in \Cref{fig:fewshot_stsb_varying_rank}, by exploiting the intrinsic low-dimensional dynamics in GD via overparameterization in width, our approach is robust to the choice of the rank $r$ in fine-tuning.
\end{itemize}

Deep LoRA only requires $3r^2$ additional trainable parameters for each adapted layer compared to vanilla LoRA, where $r$ is relatively small (e.g., $r=8$).

\subsection{More Experimental Details}\label{sec:exp}

To evaluate our approach, we use a pretrained BERT \citep{devlin2019bert} base model and apply adaptation on all attention and feedforward weights in the transformer, resulting in 72 adapted layers in total. Unless stated otherwise, we use $r=8$ for both vanilla and Deep LoRA throughout all experiments, in which case Deep LoRA has roughly 0.01\% more parameters (with respect to BERT) than vanilla LoRA. We utilize Adam \citep{kingma2014adam} as an optimizer for both methods. See \Cref{app:exp_details} for more details on the experimental setup.

\begin{table}[t]
\vspace{-4em}
\caption{\textbf{Improvement of Deep LoRA over vanilla LoRA for limited data GLUE fine-tuning.} For each task, we draw 1024 samples at random over 10 trials with different seeds, and report the performance gap (with variance) on the validation split between Deep LoRA and vanilla LoRA using the same train set.}
\begin{center}
\begin{small}
\begin{sc}
\begin{tabular}{ccccccc}
\toprule
 & \textbf{CoLA} & \textbf{MNLI} & \textbf{MRPC} & \textbf{QNLI} & \textbf{QQP} \\
 \midrule
 $\Delta$ & $+0.090_{\pm 0.002}$ & $+0.011_{\pm 0.0005}$ & $+0.0042_{\pm 0.001}$ & $+0.048_{\pm 0.0009}$ & $+0.005_{\pm 0.0002}$ \\
\bottomrule
\end{tabular}
\vskip 0.15in
\begin{tabular}{ccccc}
\toprule
& \textbf{RTE} & \textbf{SST-2} & \textbf{STS-B} & \textbf{Overall} \\
\midrule
$\Delta$ & $+0.029_{\pm 0.002}$ & $+0.019_{\pm 0.0006}$ & $+0.018_{\pm 0.00006}$ & $+0.028_{\pm 0.002}$ \\
\bottomrule
\end{tabular}
\end{sc}
\end{small}
\end{center}
\vspace{-1em}
\label{tab:deep_lora}
\end{table}

\paragraph{Advantage I: Better generalization with limited data.}
We first evaluate our approach on tasks in the GLUE benchmark \citep{wang2018glue}, which is a standard benchmark for natural language understanding. To test the performance in a limited data setting, for one given trial of a single task, we randomly sample $1024$ examples from the task data for fine-tuning, a nd compare the difference in performance on the same train set between Deep LoRA and vanilla LoRA on the entire validation split. From the results  shown in \Cref{tab:deep_lora}, we can see that Deep LoRA delivers significant improvements across most tasks compared to vanilla LoRA, and on average improves performance by nearly 3 points, a notable margin. 

This improvement in performance becomes more pronounced in scenarios with severely limited data, such as few-shot settings. Applying the same sampling procedure as in the prior study to the STS-B dataset, we assess both approaches using only $n \in \{16, 64, 256\}$ training instances. Experiments in \Cref{fig:fewshot_stsb} illustrate that Deep LoRA consistently surpasses vanilla LoRA across all sample sizes, with the most significant difference observed when $n=16$.

\paragraph{Deep LoRA finds lower rank solutions.} 
We find that at the end of fine-tuning, Deep LoRA finds lower rank solutions for $\Delta \mW_k$ than vanilla LoRA, as shown in \Cref{fig:fewshot_256_ranks}. In the limited data setting (256 samples), we see that all adapted layers in the vanilla LoRA saturate the constrained numerical rank $r=8$, while most layers in Deep LoRA are perturbed by matrices with numerical ranks between 0 and 4.\footnote{A majority of them are in fact zero, i.e., no change from pretrained weights.} This suggests that Deep LoRA can adaptively select the appropriate rank for each layer depending on the task. This low-rank bias induces implicit regularization during the fine-tuning process and ultimately prevents overfitting to the task, particularly when only few training samples are available. As a practical consideration, Deep LoRA also requires a fraction of the memory cost to store compared to vanilla LoRA due to the parsimony in adapted weights. 

\paragraph{Advantage II: Robustness to choice of rank $r$.} Due to the scarcity of the target training data, choosing the rank $r$ in LoRA is a delicate process -- it needs to be large enough to capture the complexity in modeling the downstream task, while small enough to prevent overfitting. The proposed Deep LoRA, on the other hand, avoids catastrophic overfitting as we increase $r$, as demonstrated in \Cref{fig:fewshot_stsb_varying_rank}. This observation mirrors the behavior seen in deep matrix completion, as illustrated in  \Cref{fig:depth_2_v_3}. For shallow factorizations, an overestimation of rank $r$ leads to an increase in the generalization error. In contrast, deep factorizations remain resilient to overfitting.

Finally, we show that Deep LoRA outperforms vanilla LoRA for few-shot natural language generation fine-tuning in \Cref{app:nlg}. We also provide an ablation study in \Cref{app:comp_random_lora} on the compression mechanism for Deep LoRA and show that it is crucial for accelerating training.

\section{Conclusion \& Future Directions}\label{sec:conclusion}
In this work, we have provided an in-depth exploration of low-dimensionality and compressibility in the dynamics of deep overparameterized learning, providing theoretical understandings in the setting of deep matrix factorization and applications to efficient deep matrix completion and language model adaptation. Finally, we outline a couple potential future directions following this work.

\paragraph{Compressibility in non-linear settings.} Although the results on network compression in \Cref{sec:theory} exploit the specific gradient structure of deep matrix factorizations, we believe that our analysis can provide meaningful direction for analyzing the fully non-linear case.

To sketch an idea, consider the setting of \Cref{subsec:setup} except with a \emph{non-linear} factorization, i.e., \eqref{eqn:deep-matrix} becomes
\begin{equation}
    f(\bm \Theta) := \mW_L \sigma(\mW_{L-1} \cdots \sigma(\mW_2 \sigma(\mW_1)))
\end{equation} 
where $\sigma$ is (for example) the entry-wise ReLU activation. For concreteness, consider the $L=3$ case. The gradient of the loss with respect to, e.g., $\mW_2$ in \eqref{eq:l2_loss} is given by
\begin{equation*}
    \nabla_{\mW_2} \ell(\bm \Theta) = [h(\mW_2 \sigma(\mW_1)) \odot (\mW_3^\top \bm E)] \sigma(\mW_1)^\top
\end{equation*}
where $h$ is the entry-wise unit step function and $\bm E = f(\bm \Theta) - \bm \Phi$. Comparing this to the gradient in the linear setting \eqref{eq:grad}, there is a great deal of shared structure, with the two main differences being the non-linearity applied to $\mW_1$ in the post factor and a projection on the inner term via $h(\mW_2 \sigma(\mW_1))$. However, we still have the low-rank structure of $\mW_3^\top \bm E$, and the zeroing out of certain entries preserves approximate spectral properties of the matrix \citep{chatterjee2015matrix}. Moreover, this projection is akin to the masking via $\bm \Omega$ as in deep matrix completion from \Cref{sec:dmc}, for which we do find compressible dynamics. In \Cref{fig:spectra_w2}, we plot the singular value spectrum of the above gradient at small initialization, finding that the top $r^*$ singular values separate from the rest of the spectrum. This suggests that we may be able to identify a low-dimensional subspace along which we can achieve similar dynamics to the full parameter space.

\begin{figure}[t]
    \centering
    \vspace{-4em}
    \includegraphics[width=0.5\linewidth]{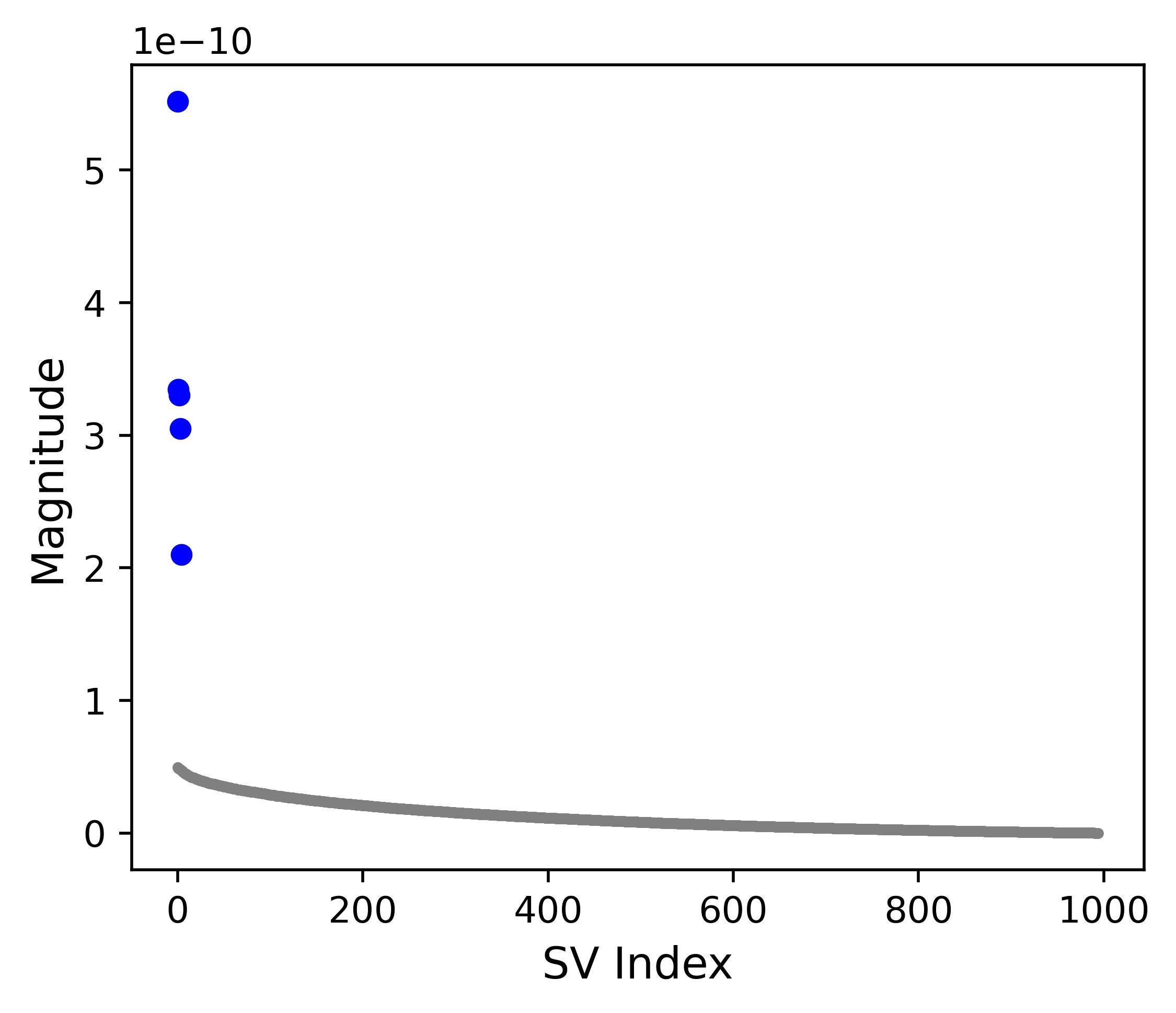}
    \caption{\textbf{Low-rank gradient of non-linear factorizations at initialization.} Singular values spectrum of 
    $\nabla_{\mW_2} \ell(\bm \Theta)$ at initialization for non-linear factorization with $L=3$, $d=1000$, $r^* = 5$, and $\epsilon_l = 10^{-3}$. The top 5 singular values separate from the tail of the spectrum.}
    \label{fig:spectra_w2}
\end{figure}

\paragraph{Extensions to Deep LoRA.} We have demonstrated the efficacy of Deep LoRA for natural language understanding and generation in \Cref{sec:dclora} and \Cref{app:nlg} respectively. However, it would be meaningful to evaluate Deep LoRA in other modalities, e.g., diffusion models, where fine-tuning on limited data is commonplace. Moreover, the high degree of alignment at initialization to the final adapted subspaces shown in \Cref{fig:low_rank_subspace} suggests that SGD (rather than Adam) can be used for the outer factors of Deep LoRA, further reducing memory costs. Finally, exploring the use of second-order methods to accelerate fine-tuning along the rank-$r$ subspace could be a potential improvement.

\paragraph{Implications for representation learning.} The low-rank bias in the end-to-end features of deep networks may have important connections to emergent phenomena in representation learning, such as \emph{deep neural collapse} \citep{zhai2024investigating,zhou2022all,yaras2022neural,wang2022linear,zhou2022optimization,zhu2021geometric,beaglehole2024average,li2024neural}, whereby the last-layer representations exhibit surprisingly simple structures. Moreover, by uncovering the low-rank evolution of \emph{individual} weights, we could shed light on more intricate phenomena such as \emph{progressive neural collapse} \citep{he2023law,wang2023understanding}.

\section*{Acknowledgement}
The work of L.B., P.W., and C.Y. were supported in part by DoE award DE-SC0022186, ARO YIP award W911NF1910027, and NSF CAREER award CCF-1845076. Q.Q., P.W., and C.Y. acknowledge support from ONR N00014-22-1-2529 and NSF CAREER CCF-214390. Q.Q. also acknowledges support from NSF CAREER CCF-2143904, NSF CCF-2212066, NSF CCF-2212326, and NSF IIS 2312842, an AWS AI Award, a gift grant from KLA, and MICDE Catalyst Grant.

{\small 
\bibliographystyle{plainnat}
\bibliography{biblio/large_models}}

\appendix
\newpage



\section{Related Works \& Future Directions}\label{app:related_works}

\paragraph{Implicit regularization.} The first work to theoretically justify implicit regularization in matrix factorization \citep{gunasekar2017implicit} was inspired by empirical work that looked into the implicit bias in deep learning \citep{neyshabur2015search} and made the connection with factorization approaches as deep nets using linear activations\footnote{Of course linear activation has been considered throughout the history of artificial neural nets, but the fact that a multilayer network with linear activation has an equivalent one-layer network meant this architecture was summarily ignored. This is evidenced in \citep{dalton1991artificial}: ``In summary, it makes no sense to use a
multilayered neural network when linear activation functions are used.''}. 
Since then a long line of literature has investigated deep factorizations and their low-rank bias, including \citep{arora2019implicit, moroshko2020implicit,timor2023implicit}; 
in fact there is so much work in this direction that there is already a survey in the Communications of the ACM \citep{vardi2023implicit}. 

Several older works explicitly imposed low-rank factorization in deep networks
\citep{jaderberg2014speeding, sainath2013low} or studied a low-rank factorization of the weights after the learning process \citep{denil2013predicting}. Newer works along these lines discuss initialization and relationships to regularization \citep{khodak2020initialization}. 

The work in \citet{oymak2019generalization} also discusses low-rank learning in deep nets, by studying the singular vectors of the Jacobian and arguing that the ``information space" or top singular vectors of the Jacobian are learned quickly. 
Very recent work has shown that the typical factorization of the weights in an attention layer of a transformer into key and query layers has an implicit bias towards low-rank weights \citep{tarzanagh2023transformers}. 

\paragraph{Overparameterization.} There is a sizeable body of work discussing the various benefits of overparameterization in deep learning settings, of which we discuss a few. \citet{du2019width} demonstrate that width is provably necessary to guarantee convergence of deep linear networks. \citet{arora2018optimization} show that overparameterization can result in an implicit acceleration in optimization dynamics for training deep linear networks. \citet{allen2019convergence} argue that overparameterization plays a fundamental role in rigorously showing that deep networks find global solutions in polynomial time. \citet{arpit2019benefits} shows that depth in ReLU networks improves a certain lower bound on the preservation of variance in information flowing through the network in the form of activations and gradients. 

\paragraph{LoRA and its variants.} There is a substantial body of existing work in the realm of parameter efficient fine-tuning -- see \citet{xu2023parameter} for a comprehensive survey. However, the method that has arguably gained the most traction in recent years is LoRA \citep{hu2021lora}, in which trainable rank decomposition matrices are added on top of frozen pretrained weights to adapt transformers to downstream tasks. Since then, there has been a plethora of variants. Generalized LoRA \citep{chavan2023one} proposes a general formulation of LoRA that encapsulates a handful of other parameter efficient adapters. AdaLoRA \citep{zhang2022adaptive} parameterizes the updates in an SVD-like form and iteratively prunes the inner factor to dynamically control the factorization rank. VeRA \citep{kopiczko2024vera} parameterizes the updates via diagonal adapters which are transformed via random projections that are shared across layers.

The idea of LoRA was initially inspired by the notion of low \textit{intrinsic dimension} of fine-tuning pretrained models. Intrinsic dimension for an objective was first proposed in \citet{li2018measuring}, where it was defined to be the minimum dimensionality needed for a random subspace to contain a solution to the objective. Using this idea, \citet{aghajanyan2021intrinsic} demonstrated that the objective of fine-tuning a pretrained model has a low intrinsic dimension. Building on this, \citet{zhang2023fine} learns the intrinsic subspace of a given fine-tune task from the original parameter trajectory to investigate transferability between these task-specific subspaces.

\section{Experimental Details}\label{app:exp_details}
The pretrained BERT and T5 models are retrieved from the HuggingFace \texttt{transformers} library \citep{wolf2019huggingface} as \texttt{google-bert/bert-base-cased} and \texttt{google-t5/t5-base} respectively. We choose the best learning rate for each method from $\eta \in \{10^{-5}, 10^{-4}, 10^{-3}, 10^{-2}\}$ on STS-B with 1024 samples, and find that $\eta = 10^{-4}$ and $\alpha=8$ works best for vanilla LoRA, while $\eta = 10^{-2}$ with $\gamma = 10^{-2}$ works best for Deep LoRA (although $\gamma$ can be chosen relatively freely). We tried using a linear decay learning rate but found worse results in the limited data setting for both vanilla and Deep LoRA. We use a maximum sequence length of 128 tokens for all tasks. Vanilla LoRA is initialized in the same fashion as the original paper (i.e., $\bm W_k^{(2)}$ is initialized to all zeros, $\bm W_k^{(1)}$ is initialized to be Gaussian with standard deviation 1), whereas Deep LoRA is compressed from a full-width 3-layer factorization with orthogonal initialization of scale $\epsilon_l = 10^{-3}$. We use a train batch size of 16, and train all models until convergence in train loss, and use the final model checkpoint for evaluation. For generative tasks, we use beam search \cite{freitag2017beam} with beam size 4 and maximum generation length of 64. All experiments are carried out on a single NVIDIA Tesla V100 GPU, with time and memory usage reported in \Cref{tab:runtimes}.

\begin{table*}[ht]
\caption{Comparison of step wall-time and memory usage for vanilla and Deep LoRA.}
\begin{center}
\begin{small}
\begin{sc}
\begin{tabular}{lcc}
\toprule
\textbf{Method} & \textbf{Iteration Wall-Time (ms)} & \textbf{Memory Usage (GB)} \\
\midrule
Vanilla LoRA & 102 & 12.526 \\
Deep LoRA & 106 & 12.648 \\
\bottomrule
\end{tabular}
\end{sc}
\end{small}
\end{center}
\label{tab:runtimes}
\end{table*}

\section{Evaluation on Natural Language Generation}\label{app:nlg}

In addition to the natural language understanding tasks evaluated in \Cref{sec:dclora}, we test the effectiveness of Deep LoRA compared to vanilla LoRA for few-shot fine-tuning for natural language generation (NLG), specifically on the E2E dataset \citep{novikova2017e2e} with the T5 base model \citep{raffel2020exploring}. All hyperparameters are as reported in \Cref{sec:dclora} and \Cref{app:exp_details}. The results are shown in \Cref{tab:nlg}. We observe significant improvements using Deep LoRA in BLEU \citep{Papineni02bleu} and ROUGE \citep{lin2004rouge} scores, with marginally worse results with respect to METEOR \citep{banarjee2005} score.

\begin{table*}[ht]
\caption{\textbf{Improvement of Deep LoRA over vanilla LoRA for few-shot NLG fine-tuning.} On the E2E dataset, we draw 16 samples at random over 10 trials with different seeds, and report the average performance gap on the validation split between Deep LoRA and vanilla LoRA for various metrics using the same train set.}
\begin{center}
\begin{small}
\begin{sc}
\begin{tabular}{ccccccc}
\toprule
 & \textbf{BLEU} & \textbf{ROUGE-1} & \textbf{ROUGE-2} & \textbf{ROUGE-L} & \textbf{ROUGE-LSUM} & \textbf{METEOR} \\
 \midrule
 $\Delta$ & $+0.033$ & $+0.032$ & $+0.056$ & $+0.061$ & $+0.047$ & $-0.00036$ \\
\bottomrule
\end{tabular}
\end{sc}
\end{small}
\end{center}
\label{tab:nlg}
\end{table*}

\section{Ablating Compression Mechanism}\label{app:ablate_comp}

\subsection{Deep Matrix Completion}\label{app:comp_random_dmc}

\begin{figure}[ht]
    \centering
    \includegraphics[width=0.7\linewidth]{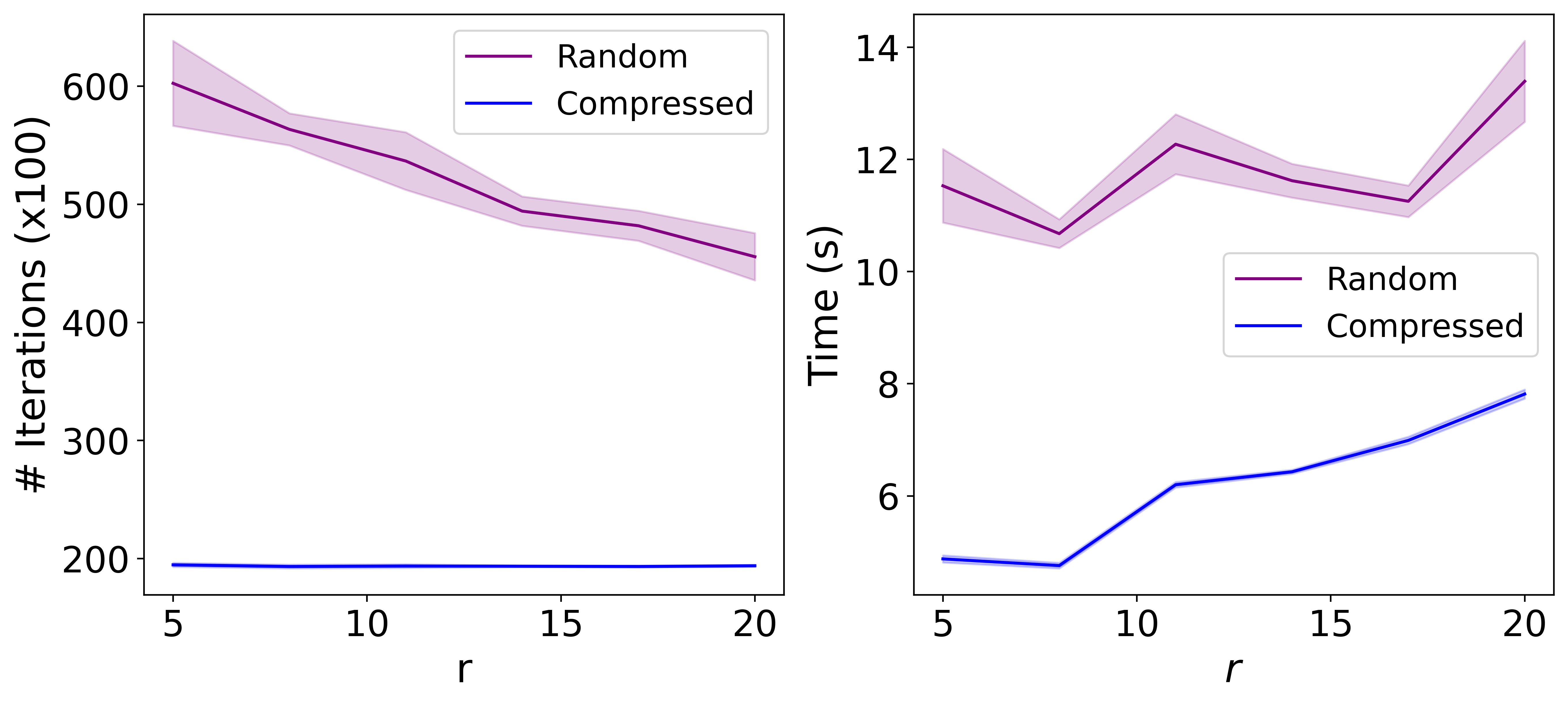}
    \caption{\textbf{Comparing efficiency of compressed networks vs. randomly initialized narrow networks for deep matrix completion} with different overestimated $r$ and $L=3$, $d=1000$, $r^* = 5$, $\epsilon_l=10^{-3}$ and 20\% of entries observed. \textit{Left}: Number of iterations to converge. \textit{Right}: Wall-time to converge.}
    \label{fig:comp_iter_time}
\end{figure}

We compare the training efficiency of deep $2r$-compressed factorizations (within a wide network of width $d \gg r$) with randomly initialized deep factorizations of width $2r$. As depicted in \Cref{fig:comp_iter_time} (left), the compressed factorization requires fewer iterations to reach convergence, and the number of iterations necessary is almost unaffected by $r$. Consequently, training compressed factorizations is considerably more time-efficient than training narrow networks of the same size, provided that $r$ is not significantly larger than $r^*$. The distinction between compressed and narrow factorizations underscores the benefits of wide factorizations, as previously demonstrated and discussed in \Cref{fig:depth_2_v_3} (right), where increasing the width results in faster convergence. However, increasing the width alone also increases computational costs -- by employing compression, we can achieve the best of both worlds.

\subsection{Deep LoRA}\label{app:comp_random_lora}

\begin{figure}[ht]
\centering
\includegraphics[width=0.7\linewidth]{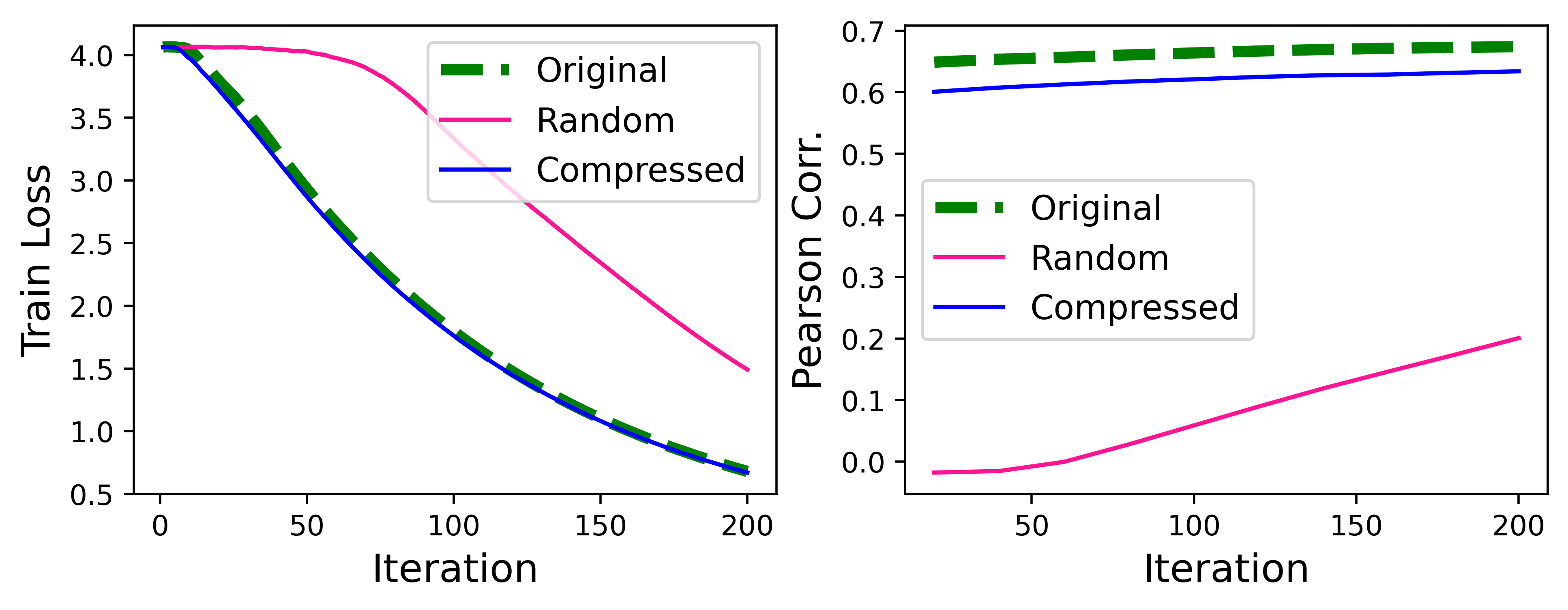}
\vspace{-0.2in}
\caption{\textbf{Compression enables faster convergence of Deep LoRA.} We compare full-width, compressed, and narrow deep factorizations for adapting to STS-B with 16 samples. \textit{Left}: Batch train loss vs. iterations. \textit{Right}: STS-B evaluation metric (Pearson correlation) vs. iterations. }
\label{fig:narrow_vs_wide_lora}
\end{figure}

We verify that compression is crucial for the efficiency of Deep LoRA. We compare the performance of three different approaches: (i) \textit{Original}, where we use a three-layer full-width factorization as in \eqref{eq:deep_lora}, (ii) \textit{Compressed}, which is the rank-$r$ compression of \eqref{eq:deep_lora} (a.k.a. Deep LoRA), and (iii) \textit{Random}, where the $\mW_k^{(i)}$ in \eqref{eq:deep_lora} are initialized randomly with $\mW_k^{(2)} \in \R^{r \times r}$.  We can see that via compression, Deep LoRA can achieve similar convergence behavior to the original overparameterized factorization with much fewer parameters, while the randomly initialized version takes much longer to train, similar to the result for deep matrix completion in \Cref{app:comp_random_dmc}.

\section{Proofs}\label{app:proofs}

The analytic form of the gradient $\nabla_{\bm \Theta}\,\ell(\bm \Theta)$ is given by
\begin{equation}\label{eq:grad}
    \nabla_{\mW_l} \ell(\bm \Theta) = \mW_{L:l+1}^\top \bm E \mW_{l-1:1}^\top, \; l \in [L]
\end{equation}
where $\bm E = f(\bm \Theta) - \bm \Phi$, which when substituted into \eqref{eq:gd} gives the update rules
\begin{equation}\label{eq:update}
    \mW_l(t+1) = (1 - \eta \lambda) \mW_l(t) - \eta \mW_{L:l+1}^\top(t) \bm E(t) \mW_{l-1:1}^\top(t), \; l \in [L]
\end{equation}
for $t = 0, 1, 2, \dots$, where $\bm E(t) = f(\bm \Theta(t)) - \bm \Phi$.

We first establish the following \Cref{lem:1} -- the claim in \Cref{thm:1} then follows in a relatively straightforward manner. We note that all statements quantified by $i$ in this section implicity hold for all $i \in [m]$ (as defined in \Cref{thm:1}) for the sake of notational brevity. 

\subsection{Proof of \Cref{thm:1}}\label{app:proof_thm1}

\begin{lemma}\label{lem:1}
Under the setting of \Cref{thm:1}, there exist orthonormal sets $\{\bm u_i^{(l)}\}_{i=1}^m \subset \R^d$ and $\{\bm v_i^{(l)}\}_{i=1}^m \subset \R^d$ for $l \in [L]$ satisfying $\bm v_i^{(l+1)} = \bm u_i^{(l)}$ for all $l \in [L-1]$ such that the following hold for all $t \geq 0$: 
\begin{alignat*}{2}
    \mathcal{A}(t)&: \mW_{l}(t) \bm v_i^{(l)} = \rho_l(t) \bm u_i^{(l)} \quad &&\forall l \in [L], \\
    \mathcal{B}(t)&: \mW_{l}^\top(t) \bm u_i^{(l)} = \rho_l(t) \bm v_i^{(l)} \quad &&\forall l \in [L], \\
    \mathcal{C}(t)&: \bm \Phi^\top \mW_{L:l+1}(t) \bm u_i^{(l)} = \bm 0 \quad && \forall l \in [L], \\
    \mathcal{D}(t)&: \bm \Phi \mW^\top_{l-1:1}(t)\bm v_i^{(l)} = \bm 0 \quad && \forall l \in [L],
\end{alignat*}
where $\rho_l(t) = \rho_l(t-1)\cdot (1 - \eta \lambda - \eta \cdot \prod_{k\neq l} \rho_k(t-1)^2)$ for all $t \geq 1$ with $\rho_l(0) = \epsilon_l > 0$.
\end{lemma}

\begin{proof}
Define $\bm \Psi := \mW^\top_{L:2}(0)\bm \Phi$. Since the rank of $\bm \Phi$ is at most $r$, we have that the rank of $\bm \Psi \in \R^{d \times d}$ is at most $r$, which implies that $\mathrm{dim} \,\mathcal{N}\left(\bm \Psi\right) = \mathrm{dim}\,\mathcal{N}\left(\bm \Psi^\top \right) \ge d - r$. We define the subspace
\begin{align*}
    \mathcal S := \mathcal N\left(\bm \Psi\right)\cap \mathcal N\left(\bm \Psi^\top \bm{W}_1(0)\right) \subset \R^d.
\end{align*}
Since $\bm{W}_1(0) \in \R^{d\times d}$ is nonsingular, we have
\begin{align*}
    \mathrm{dim}(\mathcal S) \geq 2(d-r) - d = m. 
\end{align*}
Let $\{\bm v_i^{(1)}\}_{i=1}^m$ denote an orthonormal set contained in $\mathcal S$ and set $\bm{u}_i^{(1)} := \mW_1(0)\bm v_i^{(1)}/\epsilon_1$, where $\epsilon_1 > 0$ is the scale of $\mW_1(0)$ -- since $\mW_1(0) / \epsilon_1$ is orthogonal, $\{\bm u_i^{(1)}\}_{i=1}^m$ is also an orthonormal set. Then we trivially have $\mW_1(0)\bm v_i^{(1)} = \epsilon_1 \bm{u}_i^{(1)}$, which implies $\mW_1^\top(0) \bm u_i^{(1)} = \epsilon_1 \bm v_i^{(1)}$. It follows from $\bm{v}_i^{(1)} \in \mathcal S$ that $\bm \Psi \bm{v}_i^{(1)} = \bm{0}$ and $\bm \Psi^\top \mW_1(0) \bm{v}_i^{(1)} = \bm{0}$, which is equivalent to $\mW^\top_{L:2}(0)\bm \Phi \bm v_i^{(1)} = \bm 0$ and $\bm\Phi^\top \mW_{L:2}(0) \mW_1(0) \bm v_i^{(1)} = \epsilon_1 \bm\Phi^\top \mW_{L:2}(0) \bm u_i^{(1)} = \bm 0$ respectively. Since $\mW_{L:2}^\top(0)$ is full column rank, we further have that $\bm \Phi \bm v_i^{(1)} = \bm 0$.

Now let $\mathcal E(l)$ denote that we have orthonormal sets $\{\bm u_i^{(l)}\}_{i=1}^m$ and $\{\bm v_i^{(l)}\}_{i=1}^m$ satisfying $\mW_l(0) \bm v_i^{(l)} = \epsilon_l \bm u_i^{(l)}$, $\mW_l^\top(0) \bm u_i^{(l)} = \epsilon_l \bm v_i^{(l)}$, $\bm \Phi^\top \mW_{L:l+1}(0) \bm u_i^{(l)} = \bm 0$, and $\bm \Phi \mW_{l-1:1}^\top(0) \bm v_i^{(l)} = \bm 0$. From the above arguments, we have that $\mathcal E(1)$ holds -- now suppose $\mathcal E(k)$ holds for some $1 \le k < L$. Set $\bm v_i^{(k+1)} := \bm u_i^{(k)}$ and $\bm u_i^{(k+1)} := \mW_{k+1}(0) \bm v_i^{(k+1)}/\epsilon_{k+1}$. This implies that $\mW_{k+1}(0) \bm v_i^{(k+1)} = \epsilon_{k+1} \bm u_i^{(k+1)}$ and $\mW_{k+1}^\top(0) \bm u_i^{(k+1)} = \epsilon_{k+1} \bm v_i^{(k+1)}$. Moreover, we have 
\begin{align*}
    \bm \Phi^\top \mW_{L:(k+1)+1}(0) \bm u_i^{(k+1)} &= \bm \Phi^\top\mW_{L:k+1}(0) \mW_{k+1}^\top(0) \bm u_i^{(k+1)} / \epsilon_{k+1}^2 \\
    &= \bm \Phi^\top\mW_{L:k+1}(0) \bm v_i^{(k+1)} / \epsilon_{k+1} \\
    &= \bm \Phi^\top\mW_{L:k+1}(0) \bm u_i^{(k)} / \epsilon_{k+1}  = \bm 0,
\end{align*}
where the first two equalities follow from orthogonality and $\bm u_i^{(k+1)} = \mW_{k+1}(0) \bm v_i^{(k+1)}/\epsilon_{k+1}$, and the last equality is due to $\bm v_i^{(k+1)} = \bm u_i^{(k)}$. Similarly, we have
\begin{align*}
    \bm \Phi \mW_{(k+1)-1:1}^\top(0) \bm v_i^{(k+1)} & = \bm \Phi \mW_{k-1:1}^\top(0) \mW_k^\top(0) \bm v_i^{(k+1)}  \\
    &= \bm \Phi \mW_{k-1:1}^\top(0)  \mW_k^\top(0) \bm u_i^{(k)} \\
    & = \epsilon_k \bm \Phi \mW_{k-1:1}^\top(0) \bm v_i^{(k)} = \bm 0,
\end{align*}
where the second equality follows from $\bm v_i^{(k+1)} = \bm u_i^{(k)}$ and the third equality is due to $\mW_k^\top(0)\bm u_i^{(k)} = \epsilon_k \bm v_i^{(k)}$. Therefore $\mathcal E(k+1)$ holds, so we have $\mathcal E(l)$ for all $l \in [L]$. As a result, we have shown the base cases $\mathcal A(0)$, $\mathcal B(0)$, $\mathcal C(0)$, and $\mathcal D(0)$. 

Now we proceed by induction on $t \ge 0$. Suppose that $\mathcal A(t)$, $\mathcal B(t)$, $\mathcal C(t)$, and $\mathcal D(t)$ hold for some $t \ge 0$. First, we show $\mathcal A(t+1)$ and $\mathcal B(t+1)$. We have
\begin{align*}
    \mW_l(t+1) \bm v_i^{(l)}  & =  \left[(1-\eta\lambda) \mW_l(t) -  \eta \mW_{L:l+1}^\top (t) \bm E(t)\mW_{l-1:1}^\top (t) \right]\bm v_i^{(l)} \\
    &= \left[(1-\eta\lambda) \mW_l(t) -  \eta \mW_{L:l+1}^\top (t)\left(\mW_{L:1}(t) - \bm \Phi\right)\mW_{l-1:1}^\top (t) \right]\bm v_i^{(l)} \\
    & = (1-\eta\lambda)\bm{W}_l(t) \bm v_i^{(l)} - \eta \mW_{L:l+1}^\top(t) \mW_{L:1}(t) \mW_{l-1:1}^\top(t) \bm v_i^{(l)} \\
    & = (1-\eta\lambda)\bm{W}_l(t) \bm v_i^{(l)} - \eta \cdot (\prod_{k\neq l} \rho_k^2(t)) \bm{W}_l(t) \bm v_i^{(l)} \\
    &= \rho_l(t) \cdot (1 - \eta \lambda - \eta \cdot \prod_{k\neq l} \rho_k^2(t)) \bm u_i^{(l)} = \rho_l(t+1) \bm u_i^{(l)}
\end{align*}
for all $l \in [L]$, where the first equality follows from \eqref{eq:update}, the second equality follows from definition of $\bm E(t)$, the third equality follows from $\mathcal D(t)$, and the fourth equality follows from $\mathcal A(t)$ and $\mathcal B(t)$ applied repeatedly along with $\bm v_i^{(l+1)} = \bm u_i^{(l)}$ for all $l \in [L-1]$, proving $\mathcal A(t+1)$. Similarly, we have 
\begin{align*}
    \mW_l^\top(t+1) \bm u_i^{(l)}  & = \left[(1-\eta\lambda) \mW_l^\top(t) -  \eta \mW_{l-1:1}(t) \bm E^\top(t) \mW_{L:l+1}(t) \right]\bm u_i^{(l)} \\
    &= \left[(1-\eta\lambda) \mW_l^\top(t) -  \eta \mW_{l-1:1}(t) \bm \left(\mW_{L:1}^\top(t) - \bm \Phi^\top\right) \mW_{L:l+1}(t) \right]\bm u_i^{(l)} \\
    &= (1-\eta\lambda) \mW_l^\top(t) \bm u_i^{(l)} - \eta \mW_{l-1:1}(t) \mW_{L:1}^\top(t) \mW_{L:l+1}(t) \bm u_i^{(l)} \\
    &= (1-\eta\lambda) \mW_l^\top(t) \bm u_i^{(l)} - \eta \cdot (\prod_{k\neq l} \rho_k^2(t)) \mW_l^\top(t) \bm u_i^{(l)} \\
    &= \rho_l(t) \cdot (1 - \eta \lambda - \eta \cdot \prod_{k\neq l} \rho_k^2(t)) \bm v_i^{(l)} = \rho_l(t+1) \bm v_i^{(l)}
\end{align*}
for all $l \in [L]$, where the third equality follows from $\mathcal C(t)$, and the fourth equality follows from $\mathcal A(t)$ and $\mathcal B(t)$ applied repeatedly along with $\bm v_i^{(l+1)} = \bm u_i^{(l)}$ for all $l \in [L-1]$, proving $\mathcal B(t+1)$. Now, we show $\mathcal C(t+1)$. For any $k \in [L-1]$, it follows from $\bm v_i^{(k+1)} = \bm u_i^{(k)}$ and $\mathcal A(t+1)$ that 
\begin{equation*}
    \mW_{k+1}(t+1) \bm u_i^{(k)} = \mW_{k+1}(t+1) \bm v_i^{(k+1)}  = \rho_{k+1}(t+1) \bm u_i^{(k+1)}.
\end{equation*}
Repeatedly applying the above equality for $k = l, l+1, \dots, L-1$, we obtain
\begin{equation*}
          \bm \Phi^\top \mW_{L:l+1}(t) \bm u_i^{(l)} = \left[\prod_{k=l}^{L-1} \rho_{k+1}(t)\right]\cdot \bm \Phi^\top \bm u_i^{(L)} = \bm 0
\end{equation*}
which follows from $\mathcal C(t)$, proving $\mathcal C(t+1)$. Finally, we show $\mathcal D(t+1)$. For any $k \in \{2, \dots, L\}$, it follows from $\bm v_i^{(k)} = \bm u_i^{(k-1)}$ and $\mathcal B(t+1)$ that 
\begin{equation*}
    \mW_{k-1}^\top(t+1) \bm v_i^{(k)} = \mW_{k-1}^\top(t+1) \bm u_i^{(k-1)} = \rho_{k-1}(t+1) \bm v_i^{(k-1)}.  
\end{equation*}
Repeatedly applying the above equality for $k = l, l-1, \dots, 2$, we obtain
\begin{equation*}
    \bm \Phi \mW^\top_{l-1:1}(t)\bm v_i^{(l)} = \left[\prod_{k=2}^l \rho_{k-1}(t)\right] \cdot \bm \Phi \bm v_i^{(1)} = \bm 0
\end{equation*}
which follows from $\mathcal{D}(t)$. Thus we have proven $\mathcal D(t+1)$, concluding the proof. 
\end{proof}

\begin{proof}[Proof of \Cref{thm:1}]
By $\mathcal{A}(t)$ and $\mathcal{B}(t)$ of \Cref{lem:1}, there exists orthonormal matrices $\{ \bm U_{l, 2} \}_{l=1}^L \subset \R^{d \times m}$ and $\{ \bm V_{l, 2} \}_{l=1}^L \subset \R^{d \times m}$ for $l \in [L]$ satisfying $\bm V_{l+1, 2} = \bm U_{l, 2}$ for all $l \in [L-1]$ as well as 
\begin{equation}\label{eq:svd_22}
    \bm W_l(t) \bm V_{l, 2} = \rho_l(t) \bm U_{l, 2} \quad \mbox{and} \quad \bm W_l(t)^\top \bm U_{l, 2} = \rho_l(t) \bm V_{l, 2}
\end{equation}
for all $l \in [L]$ and $t \geq 0$, where $\rho_l(t)$ satisfies \eqref{eq:rho} for $t \geq 1$ with $\rho_l(0) = \epsilon_l$. First, complete $\bm V_{1, 2}$ to an orthonormal basis for $\R^d$ as $\bm V_1 = [ \bm V_{1, 1} \ \bm V_{1, 2} ]$. Then for each $l \in [L-1]$, set $\bm U_l = [ \bm U_{l, 1} \ \bm U_{l, 2} ]$ where $\bm U_{l, 1} = \bm W_l(0) \bm V_{l, 1} / \epsilon_l$ and $\bm V_{l+1} = [ \bm V_{l+1, 1} \ \bm V_{l+1, 2} ]$ where $\bm V_{l+1, 1} = \bm U_{l, 1}$, and finally set $\bm U_L = [ \bm U_{L, 1} \ \bm U_{L, 2} ]$ where $\bm U_{L, 1} = \bm W_L(0) \bm V_{L, 1} / \epsilon_L$. We note that $\bm V_{l+1} = \bm U_l$ for each $l \in [L-1]$ and $\bm U_l, \bm V_l$ are orthogonal since $\bm W_{l}(0) / \epsilon_l$ is orthogonal for all $l \in [L]$. Then we have
\begin{equation}\label{eq:svd_12}
    \bm U_{l, 1}^\top \bm W_l(t) \bm V_{l, 2} = \rho_l(t) \bm U_{l, 1}^\top \bm U_{l, 2} = \bm 0
\end{equation}
for all $l \in [L]$ and $t \geq 0$, where the first equality follows from \eqref{eq:svd_22}. Similarly, we also have
\begin{equation}\label{eq:svd_21}
    \bm U_{l, 2}^\top \bm W_l(t) \bm V_{l, 1} = \rho(t) \bm V_{l, 2}^\top \bm V_{l, 1} = \bm 0
\end{equation}
for all $l \in [L]$ and $t \geq 0$, where the first equality also follows from \eqref{eq:svd_22}. Therefore, combining \eqref{eq:svd_22}, \eqref{eq:svd_12}, and \eqref{eq:svd_21} yields
\begin{equation*}
    \bm U_l^\top \bm W_l(t) \bm V_l = \begin{bmatrix} \bm U_{l, 1} & \bm U_{l, 2} \end{bmatrix}^\top \bm W_l(t) \begin{bmatrix} \bm V_{l, 1} & \bm V_{l, 2} \end{bmatrix} = \begin{bmatrix}
        \widetilde{\bm W}_l(t) & \bm 0 \\ \bm 0 & \rho_l(t) \bm I_m
    \end{bmatrix}
\end{equation*}
for all $l \in [L]$, where $\widetilde{\bm W}_l(0) = \epsilon_l \bm I_{2r}$ by construction of $\bm U_{l, 1}$. This directly implies \eqref{eq:weight_structures}, completing the proof.
\end{proof}

\subsection{Low-rank bias in \Cref{thm:1}}\label{app:learning_rate}

Here, we verify the claims following \Cref{thm:1} and give a precise characterization of the rate of decay of $\rho_l$ as given by \eqref{eq:rho} and the conditions on learning rate $\eta$ needed to achieve such behavior. These are given in the following lemma.

\begin{lemma}\label{lem:learning_rate}
In the setting of \Cref{thm:1}, suppose $0 < \epsilon_l = \epsilon \leq 1$ for all $l \in [L]$ and $0 < \eta \leq \frac{1}{\lambda + \epsilon}$. Then for all $t \geq 0$, the updates of $\rho_l(t)$ in \eqref{eq:rho} satisfy $\rho_l(t) = \rho(t)$ for some $\rho$, and
\begin{equation}\label{eq:rho_decay}
    \epsilon \cdot (1-\eta \cdot (\lambda + \epsilon))^t \leq \rho(t) \leq \epsilon \cdot (1-\eta\lambda)^t.
\end{equation}
\end{lemma}
Since $\lambda$ and $\epsilon$ are often chose to be small, the above lemma implies that a small learning rate is not required to achieve a low-rank solution. Moreover, by choice of $\eta$, when weight decay is employed (i.e., $\lambda > 0$) the above inequality implies that $\rho(t) \rightarrow 0$ as $t \rightarrow \infty$. When $\lambda = 0$, we instead have that $\rho$ is bounded by $\epsilon$.
\begin{proof}[Proof of \Cref{lem:learning_rate}]
If $\rho_l(0) = \epsilon$ for all $l \in [L]$, it is clear that $\rho_l(t) = \rho(t)$ for some $\rho$ for all $t \geq 0$, and the updates take the form
\begin{equation*}
    \rho(t) = \rho(t-1) \cdot \left[1 - \eta \cdot \left(\lambda + \rho(t-1)^{2(L-1)}\right)\right]
\end{equation*}
for each $t \geq 0$. We proceed by induction. For $t = 0$, since $\rho(0) = \epsilon$, the claim holds trivially. Now suppose \eqref{eq:rho_decay} holds for some $t \geq 0$. By choice of $\eta$, we have that $1 - \eta \cdot(\lambda + \epsilon) \geq 0$, so $\rho(t) \geq 0$. It then follows that
\begin{equation*}
    \rho(t+1) = \rho(t) \cdot \left[1 - \eta \cdot \left(\lambda + \rho(t)^{2(L-1)}\right)\right] \leq \rho(t) \cdot (1-\eta \lambda) \leq \epsilon \cdot (1-\eta \lambda)^{t+1}
\end{equation*}
by the fact that $\rho(t) \leq \epsilon \cdot (1-\eta \lambda)^t$. Next, by choice of $\eta$ and initial condition, we have that $\rho(t) \leq \epsilon$, so that
\begin{equation*}
    \rho(t+1) = \rho(t) \cdot \left[1 - \eta \cdot \left(\lambda + \rho(t)^{2(L-1)}\right)\right] \geq \rho(t) \cdot (1 - \eta \cdot (\lambda + \epsilon)) \geq \epsilon \cdot (1-\eta \cdot (\lambda + \epsilon))^{t+1}
\end{equation*}
since $\epsilon^{2(L-1)} \leq \epsilon$ by $\epsilon \leq 1$. The claim follows.
\end{proof}
\subsection{Proof of \Cref{prop:1}}\label{app:proof_prop1}

\begin{proof}
First, it follows from \Cref{thm:1} that for any $1 \leq i \leq j \leq L$ we have
\begin{equation}\label{eq:e2e}
    \mW_{j:i}(t) = \bm U_{j, 1} \widetilde{\mW}_{j:i}(t) \bm V_{i, 1}^\top + (\prod_{k=i}^j \rho_k(t)) \cdot \bm U_{j, 2} \bm V_{i, 2}^\top
\end{equation}
for all $t\geq 0$, where $\bm U_{l, 1}, \bm V_{l, 1} \in \R^{d \times 2r}$ and $\bm U_{l, 2}, \bm V_{l, 2} \in \R^{d \times m}$ are the first $2r$ and last $m$ columns of $\bm U_l, \bm V_l \in \R^{d \times d}$ respectively. 

The key claim to be shown here is that $\widehat{\mW}_l(t) = \widetilde{\mW}_l(t)$ for all $l \in [L]$ and $t \geq 0$. Afterwards, it follows straightforwardly from \eqref{eq:e2e} that
\begin{align*}
    &\left\| f(\bm \Theta(t)) - \widehat{f}(\widehat{\bm \Theta}(t), \bm U_{L, 1}, \bm V_{1, 1}) \right\|_F^2 \\
    = &\left\| \bm U_{L, 1} \widetilde{\mW}_{L:1}(t) \bm V_{1, 1}^\top + (\prod_{l=1}^L \rho_l(t)) \cdot \bm U_{L, 2} \bm V_{1, 2}^\top - \bm U_{L, 1} \widehat{\mW}_{L:1}(t) \bm V_{L, 1}^\top\right\|_F^2 \\
    = &\left\| \bm U_{L, 1}(\widetilde{\mW}_{L:1}(t) - \widehat{\mW}_{L:1}(t)) \bm V_{1, 1}^\top + (\prod_{l=1}^L \rho_l(t)) \cdot \bm U_{L, 2} \bm V_{1, 2}^\top \right\|_F^2
    = \left\|(\prod_{l=1}^L \rho_l(t)) \cdot \bm U_{L, 2} \bm V_{1, 2}^\top \right\|_F^2 \leq m \cdot \prod_{l=1}^L \epsilon_l^2.
\end{align*}
We proceed by induction. For $t = 0$, we have that
\begin{equation*}
    \widehat{\mW}_l(0) = \bm U_{l, 1}^\top \mW_l(0) \bm V_{l, 1} = \widetilde{\mW}_l(0)
\end{equation*}
for all $l \in [L]$ by \eqref{eq:e2e} and choice of initialization. 

Now suppose $\widehat{\mW}_l(t) = \widetilde{\mW}_l(t)$ for all $l \in [L]$. Comparing
\begin{align*}
    \widehat{\mW}_l(t+1) = (1 - \eta \lambda) \widehat{\mW}_l(t) - \eta \nabla_{\widehat{\mW}_l} \widehat{\ell}(\widehat{\bm \Theta}(t))
\end{align*}
with
\begin{align*}
    \widetilde{\mW}_l(t+1) &= \bm U_{l, 1}^\top \mW_l(t+1) \bm V_{l, 1} \\
    &= \bm U_{l, 1}^\top\left[(1-\eta\lambda) \mW_l(t) - \eta \nabla_{\mW_l} \ell(\bm \Theta(t))\right] \bm V_{l, 1} \\
    &= (1-\eta\lambda) \widetilde{\mW}_l(t) - \eta \bm U_{l, 1}^\top \nabla_{\mW_l} \ell(\bm \Theta(t)) \bm V_{l, 1}
\end{align*}
it suffices to show that
\begin{equation}\label{eq:grad_eq}
    \nabla_{\widehat{\mW}_l} \widehat{\ell}(\widehat{\bm \Theta}(t)) = \bm U_{l, 1}^\top \nabla_{\mW_l} \ell(\bm \Theta(t)) \bm V_{l, 1}, \; \forall l \in [L]
\end{equation}
to yield $\widehat{\mW}_l(t+1) = \widetilde{\mW}_l(t+1)$ for all $l \in [L]$. Computing the right hand side of \eqref{eq:grad_eq}, we have
\begin{align*}
    \bm U_{l, 1}^\top \nabla_{\mW_l} \ell(\bm \Theta(t)) \bm V_{l, 1} &= \bm U_{l, 1}^\top \mW^\top_{L:l+1}(t) (\mW_{L:1}(t) - \bm \Phi)\mW^\top_{l-1:1}(t)\bm V_{l, 1} \\
    &= (\mW_{L:l+1}(t) \bm U_{l, 1})^\top (\mW_{L:1}(t) - \bm \Phi) (\bm V_{l, 1}^\top \mW_{l-1:1}(t))^\top
\end{align*}
where 
\begin{align*}
    \mW_{L:l+1}(t) \bm U_{l, 1} = \left(\bm U_{L, 1} \widetilde{\mW}_{L:l+1}(t) \bm V_{l+1, 1}^\top + (\prod_{k=l+1}^L \rho_k(t)) \cdot \bm U_{L, 2} \bm V_{l+1, 2}^\top\right) \bm U_{l, 1} = \bm U_{L, 1} \widetilde{\mW}_{L:l+1}(t)
\end{align*}
by \eqref{eq:e2e} and the fact that $\bm U_l = \bm V_{l+1}$, and similarly
\begin{align*}
    \bm V_{l, 1}^\top \mW_{l-1:1}(t) = \bm V_{l, 1}^\top \left(\bm U_{l-1, 1} \widetilde{\mW}_{l-1:1}(t) \bm V_{1, 1}^\top + (\prod_{k=1}^{l-1}\rho_k(t)) \cdot \bm U_{l-1, 2} \bm V_{1, 2}^\top\right) =  \widetilde{\mW}_{l-1:1}(t) \bm V_{1, 1}^\top.
\end{align*}
We also have that
\begin{align*}
    \bm U_{L, 1}^\top (\mW_{L:1}(t) - \bm \Phi) \bm V_{1, 1} &= \bm U_{L, 1}^\top \left(\bm U_{L, 1} \widetilde{\mW}_{L:1}(t) \bm V_{1, 1}^\top + (\prod_{k=1}^L \rho_k(t)) \cdot \bm U_{L, 2} \bm V_{1, 2}^\top - \bm \Phi \right) \bm V_{1, 1} \\
    &= \widetilde{\mW}_{L:1}(t) - \bm U_{L, 1}^\top \bm \Phi \bm V_{1, 1}
\end{align*}
so putting together the previous four equalities yields
\begin{align*}
    \bm U_{l, 1}^\top \nabla_{\mW_l} \ell(\bm \Theta(t)) \bm V_{l, 1} &= (\mW_{L:l+1}(t) \bm U_{l, 1})^\top (\mW_{L:1}(t) - \bm \Phi) (\bm V_{l, 1}^\top \mW_{l-1:1}(t))^\top \\
    &= \widetilde{\mW}_{L:l+1}^\top(t) \bm U_{L, 1}^\top (\mW_{L:1}(t) - \bm \Phi) \bm V_{1, 1}\widetilde{\mW}^\top_{l-1:1}(t) \\
    &= \widetilde{\mW}_{L:l+1}^\top(t)(\widetilde{\mW}_{L:1}(t) - \bm U_{L, 1}^\top \bm \Phi \bm V_{1, 1}) \widetilde{\mW}^\top_{l-1:1}(t).
\end{align*}
On the other hand, the left hand side of \eqref{eq:grad_eq} gives
\begin{align*}
    \nabla_{\widehat{\mW}_l} \widehat{\ell}(\widehat{\bm \Theta}(t)) &= \widehat{\mW}_{L:l+1}(t)^\top \bm U_{L, 1}^\top (\bm U_{L, 1} \widehat{\mW}_{L:1}(t)\bm V_{1, 1}^\top - \bm \Phi) \bm V_{1, 1} \widehat{\mW}_{l-1:1}(t)^\top \\
    &= \widehat{\mW}_{L:l+1}(t)^\top(\widehat{\mW}_{L:1}(t) - \bm U_{L, 1}^\top \bm \Phi \bm V_{1, 1}) \widehat{\mW}_{l-1:1}(t)^\top
\end{align*}
so \eqref{eq:grad_eq} holds by the fact that $\widehat{\mW}_l(t) = \widetilde{\mW}_l(t)$ for all $l \in [L]$, completing the proof.
\end{proof}

\end{document}